%% file: neurips_2026.tex
\documentclass{article}

\PassOptionsToPackage{numbers, compress}{natbib}

\usepackage[preprint]{neurips_2026}

\usepackage[utf8]{inputenc} 
\usepackage[T1]{fontenc}    
\usepackage{hyperref}       
\usepackage{url}            
\usepackage{booktabs}       
\usepackage{amsfonts}       
\usepackage{amsmath}        
\usepackage{amssymb}        
\usepackage{amsthm}         
\usepackage{nicefrac}       
\usepackage{microtype}      
\usepackage{xcolor}         
\usepackage{algorithm}      
\usepackage{algpseudocode}  
\usepackage{graphicx}       

\newtheorem{theorem}{Theorem}
\newtheorem{corollary}{Corollary}
\newtheorem{lemma}{Lemma}

\theoremstyle{remark}

\title{DiPhon: Diffusion on Graphons for Scalable Graph Generation}

\author{%
  Sergio~Rozada \\
  King Juan Carlos University \\
  \texttt{sergio.rozada@urjc.es}\\
  \And
  Yiming~Qin \\
  EPFL, Lausanne, Switzerland \\
  \texttt{yiming.qin@epfl.ch}\\
  \And
  Manuel~Madeira \\
  EPFL, Lausanne, Switzerland \\
  \texttt{manuel.madeira@epfl.ch}\\
  \And
  Pascal~Frossard \\
  EPFL, Lausanne, Switzerland \\
  \texttt{pascal.frossard@epfl.ch}\\
  \And
  Alejandro~Ribeiro \\
  University of Pennsylvania \\
  \texttt{aribeiro@seas.upenn.edu}\\
}

\begin{document}

\maketitle

\begin{abstract}
Diffusion models represent a leading paradigm for graph generation, with notable impact in domains such as molecular design.
Yet, scaling these models to large graphs remains an open problem.
We approach this question in the dense-graph setting through the lens of graphons, the size-agnostic limit objects of dense graph sequences, to study how structural graph statistics behave across node-size scales.
This perspective leads to DiPhon, a diffusion framework for size-scalable graph generation.
Specifically, we formulate a continuous diffusion process on the graphon space via a Jacobi stochastic differential equation (SDE), and propose DiPhon, a discretized graph-level process that mimics these dynamics on finite graphs.
We further derive the corresponding reverse-time process, which requires access to the marginal score.
For the Jacobi process, this score interestingly admits a tractable form, which we estimate from data via graph denoising and plug into the reverse process to generate graph samples.
We prove that DiPhon matches exactly the first moment of the marginal distributions induced by the continuous graphon process, and approximates the second moment up to a closed-form discrepancy.
Thus, DiPhon inherits key size-agnostic statistical properties of the graphon dynamics, providing a principled route toward scalable graph generation.
Empirically, we demonstrate this scalability by training on small graphs and generating progressively larger graphs at inference time, without retraining, while preserving their core topological properties.
\end{abstract}

\input{sections/introduction.tex}

\input{sections/related_work.tex}

\input{sections/section3.tex}

\input{sections/section4.tex}

\input{sections/section5.tex}

\input{sections/section6.tex}

\input{sections/section7.tex}



\newpage

\small
\bibliographystyle{unsrtnat}
\bibliography{references}
\normalsize

\newpage


\appendix

\input{sections/appendix_related_work.tex}

\input{sections/appendix_jacobi_process.tex}

\input{sections/appendix_proof_moments_matching.tex}

\input{sections/appendix_proof_density_homomorphism.tex}

\input{sections/appendix_algorithms.tex}

\input{sections/appendix_parameterization.tex}

\input{sections/appendix_experiments.tex}

\input{sections/appendix_discussion.tex}



\end{document}

%% file: sections/introduction.tex
\section{Introduction}
\label{sec::intro}

Graph generative models based on diffusion have become a leading paradigm for sampling unknown graph distributions given access to graph samples \cite{jo2022gdss, vignac2023digress}. This has proven relevant in applications such as molecular design ~\cite{vignac2023digress,vignac2023midi} and neural architecture search~\cite{li2023graphpnas}. 
Despite this progress, scalability remains a fundamental open challenge in two complementary senses. 
First, diffusion models are costly to train and deploy in large-graph regimes. 
Second, even when training is feasible on small graphs, it remains unclear whether the learned generative mechanism can be transferred to larger graphs without retraining.

In this work, we focus on the latter challenge in the dense-graph setting leveraging the graphon framework
\cite{lovasz2006limits,lovasz2012large}. 
Graphons are the canonical limit objects of sequences of dense graphs of growing size. 
They offer a size-agnostic representation of an underlying random-graph law, and therefore a principled bridge between training and deployment regimes that differ only in the number of nodes. 
Graphons have been successfully used to establish size transferability in graph neural networks \cite{ruiz2021graphonsp}, showing that models trained on smaller graphs can generalize to larger ones when both are generated from a common graphon \cite{ruiz2020graphon}.
Our goal is to bring this graphon-based transfer perspective to the graph generative setting, through a graph diffusion-based formulation.

In this direction, the main limitation is that existing graph diffusion methods do not admit a natural formulation in the graphon setting. 
In particular, discrete graph diffusion models operate directly on categorical edge variables and therefore do not extend naturally to a continuous spatial domain on which graphons are defined \cite{vignac2023digress,qin2025defog}.
On the other hand, Gaussian stochastic differential equation (SDE) formulations inject noise pointwise, but their trajectories are unbounded and thus cannot be directly interpreted as edge probabilities \cite{jo2022gdss,jo2024grum}.
As a result, their connection to graphons remains at best implicit. 
To address this, we adapt a Jacobi-type SDE \cite{avdeyev2023dirichlet} to graph generation, since it can be applied pointwise to a graphon while remaining bounded in $[0,1]$ by construction.
We focus on binary graph generation and leave node features and categorical edges outside the scope of this work, providing a controlled setting in which to study scalability across graph sizes.
More specifically,
\begin{itemize}
    \item[\textbf{C1}] We formulate a diffusion process for graphons based on a Jacobi SDE, from which we derive DiPhon, a graph generative model obtained by discretizing the underlying graphon diffusion.
    \item[\textbf{C2}] We characterize the relation between the graphon-level process and DiPhon, showing exact agreement at the level of the first moment and providing a closed-form expression for the second-moment discrepancy.
    \item[\textbf{C3}] We introduce a sampling procedure that generates graphs from the marginal distributions of the model, and show that, in the stationary regime, the resulting graph matches the homomorphism densities of the Erd\H{o}s--R\'enyi graphon, a simple stationary law for graphs.
    \item[\textbf{C4}] We show empirically, across graph families commonly used to benchmark graph generative models, that DiPhon can be trained on a regime of small graphs and then transferred effectively to out-of-distribution larger graph sizes across multiple scenarios, providing promising evidence for scalable graph generation.
\end{itemize}
The remainder of the paper is organized as follows.
Section~\ref{sec::graphon_to_graph} introduces the forward graphon process and its discretization.
Section~\ref{sec::reverse} derives the reverse process and the associated learning problem.
Section~\ref{sec::experiments} presents numerical results validating the scalability of DiPhon.
Finally, Section~\ref{sec::limitations} discusses the regime of validity and the scope of the analysis.

%% file: sections/related_work.tex
\section{Related work}
\label{sec::related-work}

\paragraph{Graphons as a transfer object.} Graphons are the canonical limit of dense graph sequences \cite{lovasz2006limits, borgs2008convergent, lovasz2012large}, and they have been used as a quantitative bridge between graph sizes mostly for the analysis of graph neural networks \cite{ruiz2020graphon, ruiz2021graphonsp, levie2021transferability, maskey2023transferability}. Some works use graphons in generative methods, mainly i) to estimate the graphon itself using autoencoders or implicit neural representations \cite{xu2021graphon, xia2023implicit, azizpour2024scalable}, and ii) using graphons with specific motifs as priors for flow matching over graphs \cite{wijesinghe2026flowette}. No prior work places the diffusion in graphon space or studies properties that hold across graph scales.

\paragraph{Diffusion processes for graph generation.} Existing graph diffusion models fall mainly into two families. Discrete approaches \cite{austin2021d3pm, vignac2023digress, qin2025defog} place the edges in a finite discrete probability distribution and run a Markov chain on it. Thus, they do not naturally extend to the continuous spatial domain on which graphons are defined, preventing a graphon-based analysis of scalability in dense-graph regimes. Gaussian-SDE approaches \cite{niu2020edpgnn, jo2022gdss, jo2024grum} act pointwise on the adjacency entries but do not preserve the bounded range that a graphon-valued field requires. Diffusion on the simplex \cite{avdeyev2023dirichlet, stark2024dirichlet, eijkelboom2024catflow, davis2024fisher} and on more general bounded or geometric domains \cite{fishman2023constrained, debortoli2022riemannian, yim2023se3} sits in the right family for our needs, but it has been barely used for graph generation, and none of these constructions has been connected to a graphon setup. We adopt the Jacobi diffusion introduced in \cite{avdeyev2023dirichlet}, which sits in this family, and use it as the pointwise law of a graphon-valued forward process from which the graph generative model is then derived. We refer the reader to Appendix~\ref{app::related-work} for a more detailed discussion.

%% file: sections/section3.tex
\section{From graphon diffusion to graph diffusion}
\label{sec::graphon_to_graph}

\subsection{Preliminaries}
\label{sec::preliminaries}

\textbf{Graphons.} A graphon is a symmetric, measurable function $W:[0,1]^2\to[0,1]$, the natural limit of sequences of dense graphs of growing size, where $W(x,y)$ encodes the probability of an edge between two nodes labeled by $x,y\in[0,1]$. Closeness between two graphons, or between a finite graph and a graphon, is captured by homomorphism densities. For a finite simple graph $H=(V(H),E(H))$ with vertex set $V(H)$ and edge set $E(H)$, it is described by
\begin{equation}
    \label{eq::homdensity}
    t(H,W)
    \;=\;
    \int_{[0,1]^{V(H)}}
    \prod_{\{u,v\}\in E(H)} W(x_u,x_v)
    \prod_{u\in V(H)} dx_u,
\end{equation}
and for a finite graph $G_N$ on $N$ nodes, $t(H,G_N)=\mathrm{hom}(H,G_N)/N^{|V(H)|}$, with $\mathrm{hom}(H,G_N)$ being the number of adjacency-preserving maps from $V(H)$ to the node set of $G_N$. A sequence of graphs converges to $W$ when $t(H,G_N)$ goes to $t(H,W)$ for every finite $H$, which is what characterizes graphons as limits of dense graph sequences of growing node sizes~\cite{ruiz2021graphonsp}.

\textbf{Generative diffusion models.} On the other hand, score-based diffusion noises the data until its distribution converges to a simple stationary distribution, then runs the process backward so that reference samples are mapped back to the data distribution~\cite{jo2022gdss}. Let $x_t\in\mathbb{R}^d$ denote the diffused state at time $t\in[0,T]$, with $x_0$ drawn from the data and marginal density $p_t$. The forward and backward dynamics~\cite{anderson1982reverse} read
\begin{align}
    dx_t &\;=\; f(x_t,t)\,dt + \sigma(t)\,dB_t, \label{eq::sde-forward-generic}\\
    dx_t &\;=\; \bigl[f(x_t,t) - \sigma(t)^2\,\nabla_x \log p_t(x_t)\bigr]\,dt + \sigma(t)\,d\bar B_t, \label{eq::sde-reverse-generic}
\end{align}
where $f$ is a drift coefficient, $\sigma(t)\ge 0$ a time-dependent noise amplitude, and $B_t$ ($\bar B_t$) a standard Brownian motion in forward (backward) time. Both processes share the same marginals $p_t$ and differ only in the score $\nabla_x\log p_t$, which is unknown in closed form and typically learned from data via score matching~\cite{jo2022gdss}.

\subsection{The graphon forward process}
\label{sec::graphon_forward}

We open our contribution by describing the diffusion process on the graphon space. We let $W_0$ denote the initial graphon and define the random field $W_t(x,y)$ on $[0,1]^2$ through the Jacobi stochastic dynamics \cite{avdeyev2023dirichlet}
\begin{equation}\label{eq::spde}
    dW_t(x,y)
    \;=\;
    \kappa\bigl(\mu-W_t(x,y)\bigr)\,dt
    \;+\;
    \sigma\,\sqrt{W_t(x,y)\bigl(1-W_t(x,y)\bigr)}\,dB_t(x,y),
\end{equation}
where $B_t(x,y)$ denotes a space--time Brownian motion. The parameter $\mu\in(0,1)$ is the mean of the stationary law, $\kappa>0$ is the rate at which the field is pulled toward $\mu$, and $\sigma>0$ is the strength of the noise. 
The diffusion term $\sqrt{W_t(x,y)\bigl(1-W_t(x,y)\bigr)}$ vanishes at the boundaries, which keeps the process in $[0,1]$.
Equation~\eqref{eq::spde} is not well posed pointwise, since the driving noise $B_t(x,y)$ is space--time white noise, so the solution $W_t$ is generally defined only in a distributional sense rather than through pointwise values.
If one were to fix a spatial coordinate, the resulting one-dimensional evolution would match a scalar Jacobi diffusion with parameters $(\kappa,\sigma,\mu)$, whose stationary distribution is $\mathrm{Beta}(\mu)$. 
For a greater discussion of the scalar Jacobi diffusion that are used throughout the paper we refer the reader to Appendix~\ref{app::jacobi}.

\subsection{Discretizing the graphon}
\label{sec::discretization}

Ideally, one would define the forward dynamics directly on the continuous graphon. 
Then, at any time, graphs of arbitrary node size could be obtained by sampling the diffused graphon, yielding a scale-free generative process. 
However, this ideal construction would require evaluating space-time white noise pointwise, which is not well defined.
We therefore work with a discretized version of the graphon, obtained by integrating over cells of positive measure. 
The resulting cell averages are well defined and become the finite-dimensional objects on which the diffusion acts. 
This is not a practical limitation: in applications, we never observe the continuous graphon itself, but only finite graphs sampled from it.

Accessing the graphon through discretization should follow a diffuse-then-discretize paradigm, whereby the diffusion is first run in the continuous domain and the resulting trajectory is only then discretized.
This entails partitioning the unit square into cells $I_i\times I_j$ of equal area $1/N^2$, and define the cell average of the formal field as
\begin{equation}\label{eq::cell-average}
    \bar W_t^{ij}
    \;:=\;
    N^2 \int_{I_i\times I_j} W_t(x,y)\,dx\,dy.
\end{equation}

In practice, since \eqref{eq::spde} is not well defined pointwise, we do not have access to a continuous diffused graphon that could then be discretized.
We can, however, reverse the order and first discretize the graphon, then run the diffusion on the resulting finite-dimensional object. 
This yields a discrete construction that mirrors the continuous dynamics. 
The complementary route we propose instead assigns an independent scalar diffusion to each cell
\begin{equation}\label{eq::cell-sde-renorm}
    d\tilde W_t^{ij}
    \;=\;
    \tilde\kappa\bigl(\tilde\mu_t-\tilde W_t^{ij}\bigr)\,dt
    \;+\;
    \tilde\sigma\,\sqrt{\tilde W_t^{ij}\bigl(1-\tilde W_t^{ij}\bigr)}\,dB_t^{ij},
    \qquad \tilde W_0^{ij} = \bar W_0^{ij},
\end{equation}
with $\{B_t^{ij}\}$ independent standard Brownian motions and parameters $(\tilde\kappa,\tilde\sigma,\tilde\mu_t)$. 
The initial condition $\tilde W_0^{ij} = \bar W_0^{ij}$ is the cell average of the underlying graphon over the cell $I_i\times I_j$. 

The choice $(\tilde\kappa,\tilde\sigma,\tilde\mu_t)=(\kappa,\sigma,\mu)$ is valid and recovers, at an algebraic level, the same scalar Jacobi diffusion as in the continuous setting. 
We aim to show, however, that we can choose these parameters so that the discretize-then-diffuse construction preserves key properties of the continuous diffusion. 
In particular, we seek to design the discrete process such that the probability distribution of $\tilde W_t^{ij}$ maatches, as closely as possible, the probability distribution of $\bar W_t^{ij}$ induced by the diffuse-then-discretize route in \eqref{eq::spde}. 
Since a direct characterization of the probability distributions is intractable, we instead match their first and second moments. To such end, we introduce the following renormalization of the parameters.
\begin{equation}\label{eq::renorm-params}
    \tilde\sigma^2 \;:=\; \sigma^2 N^2,
    \qquad
    \tilde\kappa \;:=\; \kappa - \frac{\sigma^2\,(N^2-1)}{2},
    \qquad
    \tilde\mu_t \;:=\; \mathbb{E}[\bar W_t^{ij}] + \frac{\kappa}{\tilde\kappa}\bigl(\mu - \mathbb{E}[\bar W_t^{ij}]\bigr),
\end{equation}
where $\tilde\mu_t$ is a time-varying target whose role is to preserve the mean trajectory of $\tilde W_t^{ij}$ exactly. The expectation $\mathbb{E}[\bar W_t^{ij}]$ admits the closed form $\mu+(\bar W_0^{ij}-\mu)e^{-\kappa t}$ (see Lemma~\ref{lem::moments-continuous} of Appendix~\ref{app::moments-matching}), so $\tilde\mu_t$ is deterministic and can be evaluated at every time $t$ from the cell average $\bar W_0^{ij}$ of the initial condition. With this renormalization at hand, we can establish the following result. The proof is deferred to Appendix~\ref{app::moments-matching}.

\begin{theorem}[Moment alignment]\label{thm::moment-matching}
    Let $\bar W_t^{ij}$ be defined by \eqref{eq::cell-average} and let $\tilde W_t^{ij}$ solve \eqref{eq::cell-sde-renorm} with parameters \eqref{eq::renorm-params} and shared initial condition $\tilde W_0^{ij}=\bar W_0^{ij}$. Then for every $t\ge 0$,
    \begin{align}
        \mathbb{E}\bigl[\tilde W_t^{ij}\bigr]
        &\;=\;
        \mathbb{E}\bigl[\bar W_t^{ij}\bigr], \\[2pt]
        \mathrm{Var}\bigl(\tilde W_t^{ij}\bigr) - \mathrm{Var}\bigl(\bar W_t^{ij}\bigr)
        &\;=\;
        \beta_0^{ij}\, N^2\, e^{-2\kappa t}\left(1-e^{-\sigma^2 t}\right),
    \end{align}
    where $\beta_0^{ij}\;:=\;N^2\!\int_{I_i\times I_j}\bigl(W_0(x,y)-\bar W_0^{ij}\bigr)^2\,dx\,dy$ denotes the in-cell variance of the initial graphon $W_0$ over $I_i\times I_j$.
\end{theorem}

The first moment is matched exactly, uniformly in time. The second moment is not matched exactly, but the discrepancy is explicit and controlled by three quantities: the in-cell variance $\beta_0^{ij}$, the discretization factor $N^2$, and the time factor $e^{-2\kappa t}(1-e^{-\sigma^2 t})$. 
For regular graphons, fine discretizations make $\beta_0^{ij}$ small, since the graphon is nearly constant within each cell. In particular, $\beta_0^{ij}\to 0$ as $N\to\infty$. Consequently, under mild regularity conditions on $W_0$, the product $N^2\beta_0^{ij}$ remains of constant order, so the variance mismatch remains controlled. In time, the discrepancy is zero at $t=0$, increases up to a finite maximum, and then decays exponentially at rate $2\kappa$.
The result states that the discrete equation \eqref{eq::cell-sde-renorm} can be adjusted so that the probability distributions of $\tilde W_t^{ij}$ and $\bar W_t^{ij}$ match in mean exactly, and in variance up to a closed-form discrepancy, which is expected to be small independently of $N$. The discrete model thereby inherits the statistical properties of running the diffusion in the continuous, scale-free graphon space.

\begin{figure}[t]
    \centering
    \includegraphics[width=0.9\linewidth]{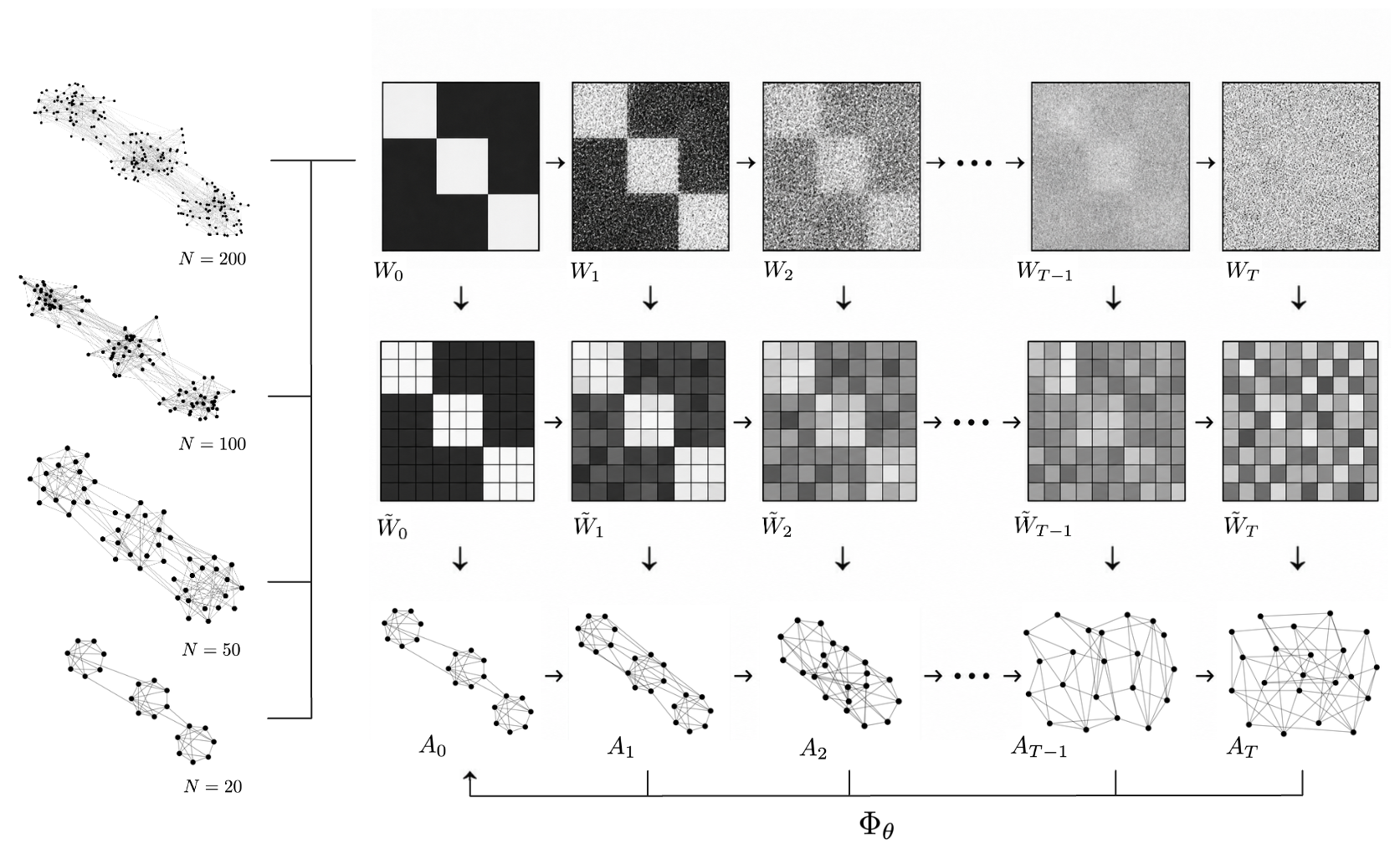}
    \caption{Schematic of DiPhon. The top row shows the original graphon diffusion process, the middle row shows the corresponding discretized graphon process, and the bottom row shows the sampled graph diffusion process on a finite graph.}
    \vspace{-4pt}
    \label{fig::diphon}
    \vspace{-4pt}
\end{figure}

\subsection{The graph as initial condition}
\label{sec::graph-as-ic}

Even though the discretized graphon is the right object to diffuse, in practice we do not observe it directly but only samples drawn from it. 
What we observe is a finite graph $A_0$, and we assume that $A_0$ is consistent with the model in the sense that each entry $A_0^{ij}\in\{0,1\}$ is a Bernoulli sample of the unobserved graphon over the cell $I_i\times I_j$, with marginal mean equal to the cell average $\bar W_0^{ij}$. The process is then run with the substitution
\begin{equation}\label{eq::ic-substitution}
    \tilde W_0^{ij}
    \;=\;
    A_0^{ij},
    \qquad
    A_0^{ij}
    \;\sim\;
    \mathrm{Ber}\bigl(\bar W_0^{ij}\bigr),
\end{equation}
and integrated forward exactly as in \eqref{eq::cell-sde-renorm}, the only change being the initial condition. 
Since the trajectories of the diffusion are themselves graphons, this substitution simply replaces a graphon initial condition by its Bernoulli sample. 
The choice mirrors standard practice in graph diffusion. Discrete diffusion~\cite{vignac2023digress}, for instance, treats the observed graph as a set of edge-probability distributions and lets the forward process turn the trajectory into a probability distribution on graphs.

The substitution also forces an adjustment of the time-varying target $\tilde\mu_t$, whose closed form $\mathbb{E}[\bar W_t^{ij}] = \mu+(\bar W_0^{ij}-\mu)e^{-\kappa t}$ depends on the unobserved cell average $\bar W_0^{ij}$. We likewise replace it by $A_0^{ij}$, so that $\tilde\mu_t$ becomes a closed-form function of $A_0^{ij}$ computable at every $t$. The two changes together, namely Bernoulli initial condition and Bernoulli-driven target, inject variability into the initial condition that propagates through the second moment, as formalized in Corollary~\ref{cor::forward-graph-ic}.

\begin{corollary}[Initial graph condition]\label{cor::forward-graph-ic}
    Under the substitution \eqref{eq::ic-substitution}, for every $t\ge 0$,
    \begin{align}
        \mathbb{E}\bigl[\tilde W_t^{ij}\bigr]
        &\;=\;
        \mathbb{E}\bigl[\bar W_t^{ij}\bigr],
        \\[4pt]
        \mathrm{Var}\bigl(\tilde W_t^{ij}\bigr) - \mathrm{Var}\bigl(\bar W_t^{ij}\bigr)
        &\;=\;
        \beta_0^{ij}\,N^2\,e^{-2\kappa t}\bigl(1-e^{-\sigma^2 t}\bigr)
        \;+\;
        \beta_1^{ij}\,e^{-2\kappa t}\,\bigl[\,1-N^2\bigl(1-e^{-\sigma^2 t}\bigr)\,\bigr],
    \end{align}
    where $\beta_1^{ij}:=\bar W_0^{ij}(1-\bar W_0^{ij})$ is the variance of the Bernoulli initial condition $A_0^{ij}$.
\end{corollary}

Again, the first-moment match is exact: the Bernoulli initial condition and the target $\tilde\mu_t$ are matched in mean, and averaging over $A_0^{ij}$ recovers the deterministic trajectory $\mathbb{E}[\bar W_t^{ij}]$. The variance gap decomposes additively into the in-cell graphon variance term of Theorem~\ref{thm::moment-matching} and an extra term proportional to the Bernoulli initial-condition variance $\beta_1^{ij}$. Both contributions share the $e^{-2\kappa t}$ term and vanish as $t\to\infty$. The proof is deferred to Appendix~\ref{app::corollary-proof}.
This result identifies the sampled graph process as a finite-resolution representative of the same graphon-level forward law.
The exact first-moment agreement defines a statistical invariant across graph sizes, while the remaining discrepancy is confined to explicit second-order terms that vanish along the forward dynamics. 
Hence, the process used on finite graphs retains the scale-agnostic structure of the graphon diffusion, providing the forward model on which the reverse construction of the next section is built.

A distinctive feature of the proposed method, in contrast to Gaussian-noise based diffusions, is that it admits sampling of graph trajectories. 
This is a direct consequence of the trajectories of our process being probability distributions over edges. Given $\tilde W_t^{ij}$ produced by \eqref{eq::cell-sde-renorm}, we generate edges via independent Bernoulli draws, one per entry,
\begin{equation}\label{eq::structured-sampling}
    A_t^{ij}
    \;\sim\;
    \mathrm{Ber}\bigl(\tilde W_t^{ij}\bigr),
    \qquad i,j\in\{1,\dots,N\}.
\end{equation}

The end-to-end forward pipeline starts from the observed graph $A_0$, uses it as the initial condition for the process via \eqref{eq::ic-substitution}, runs the SDE \eqref{eq::cell-sde-renorm} up to time $t$, and produces an edge set via \eqref{eq::structured-sampling}. The output at every time $t$ is itself a graph, so the pipeline gives access to graph-valued samples of the diffusion.

We have already shown that the probability distributions of the trajectory match those of the discretized graphon process exactly in the first moment and up to a closed-form gap in the second. Comparing the graphs sampled along the trajectory directly is more delicate. We can, however, guarantee well-behavior asymptotically. As $t\to\infty$, each $\tilde W_t^{ij}$ converges in distribution to $\mathrm{Beta}(\mu)$, and the graph $A_t$ produced by the pipeline converges in homomorphism density to the Erd\H{o}s--R\'enyi graphon at probability $\mu$ as the discretization size $N$ grows, as we prove next. The proof is given in Appendix~\ref{app::er-limit}.

\begin{theorem}[Erd\H{o}s--R\'enyi limit in homomorphism density]\label{thm::er-limit}
    Let $A_t$ denote a graph on $N$ vertices produced by the pipeline \eqref{eq::structured-sampling} on the process \eqref{eq::cell-sde-renorm}. Then for every finite simple graph $H$, the limit converges asymptotically in probability as
    \begin{equation}\label{eq::er-mean}
        \lim_{N\to\infty}\,\lim_{t \to \infty}\,t(H,A_t)
        \;=\;
        \mu^{|E(H)|},
    \end{equation}
\end{theorem}

This result shows that, in the stationary regime, every homomorphism density of the sampled graph converges to its Erd\H{o}s--R\'enyi counterpart with parameter $\mu$ as the discretization becomes finer. Accordingly, the forward diffusion approaches a generic reference distribution, independent of the data and shared across all initial graphs.

%% file: sections/section4.tex
\section{Reverting the process}
\label{sec::reverse}

Section~\ref{sec::graphon_to_graph} introduced a graph forward process that matches the graphon diffusion at the level of moments. We now turn this construction into a generative model by studying its reverse-time dynamics, characterizing the score of \eqref{eq::cell-sde-renorm}, and casting training as a denoising problem.

\subsection{Reverse-time SDE}
\label{sec::reverse-sde}

Let $\tilde p_t$ denote the marginal density of $\tilde W_t^{ij}$ at time $t\geq 0$, and fix a terminal horizon $T>0$. The forward Jacobi SDE \eqref{eq::cell-sde-renorm} admits a reverse-time SDE with the same diffusion coefficient and a drift corrected by the marginal score \cite{anderson1982reverse}, defined as
\begin{equation}\label{eq::reverse-sde}
    d\tilde W_t^{ij}
    \;=\;
    \Bigl[\,\tilde\kappa\bigl(\tilde\mu_t-\tilde W_t^{ij}\bigr)
    \;-\;
    \tilde\sigma^2\,\tilde W_t^{ij}\bigl(1-\tilde W_t^{ij}\bigr)\,\nabla_W\log \tilde p_t\bigl(\tilde W_t^{ij}\bigr)\Bigr]\,dt
    \;+\;
    \tilde\sigma\,\sqrt{\tilde W_t^{ij}\bigl(1-\tilde W_t^{ij}\bigr)}\,d\bar B_t^{ij}.
\end{equation}
Two observations are in order. First, the reverse process has the same marginals as the forward process at every \(t\) \cite{anderson1982reverse}, so the moment-matching results of Section~\ref{sec::graphon_to_graph} carry over to the backward process. Second, all terms in \eqref{eq::reverse-sde} are known except for  \(\nabla_W \log \tilde p_t\), which must therefore be estimated.

The marginal density of $\tilde W_t^{ij}$ on each cell admits no closed-form expression in the usual sense, but it does admit a spectral eigendecomposition in terms of Jacobi polynomials (see Appendix~\ref{app::jacobi-marginal}). Conditioned on the unobserved graphon entry $W_0^{ij}$, the transition density takes the form
\begin{equation}\label{eq::transition-spectral}
    \tilde p_t\bigl(\tilde W_t^{ij}\,\big|\,W_0^{ij}\bigr)
    \;=\;
    \pi\bigl(\tilde W_t^{ij}\bigr)\,
    \sum_{n=0}^{\infty}
    e^{-\lambda_n t}\,\phi_n\bigl(W_0^{ij}\bigr)\,\phi_n\bigl(\tilde W_t^{ij}\bigr),
\end{equation}
where $\pi$ is the density function of the Beta distribution. The function $\phi_n$ is the $n$-th shifted Jacobi polynomial orthonormalized with respect to $\pi$, evaluated at the corresponding state. The eigenvalue $\lambda_n=n\bigl[ \tilde \kappa+(n-1)\tilde \sigma^2/2\bigr]$ defines a decay rate of mode $n$. We defer the reader to Appendix~\ref{app::jacobi-marginal}.

The conditional score admits a closed form as well. The Jacobi polynomials satisfy a standard derivative recurrence~\cite{szego1939orthogonal} from which $\phi_n'$ is obtained explicitly, so differentiating yields
\begin{equation}\label{eq::cond-score}
    \nabla_W\log \tilde p_t\bigl(\tilde W_t^{ij}\,\big|\,W_0^{ij}\bigr)
    \;=\;
    \nabla_W\log\pi\bigl(\tilde W_t^{ij}\bigr)
    \;+\;
    \frac{\sum_{n=0}^{\infty} e^{-\lambda_n t}\,\phi_n\bigl(W_0^{ij}\bigr)\,\phi_n'\bigl(\tilde W_t^{ij}\bigr)}
         {\sum_{n=0}^{\infty} e^{-\lambda_n t}\,\phi_n\bigl(W_0^{ij}\bigr)\,\phi_n\bigl(\tilde W_t^{ij}\bigr)}.
\end{equation}

Equation~\eqref{eq::cond-score} expresses the conditional score as a closed-form function of the noisy state $\tilde W_t^{ij}$ and of the graphon entry $W_0^{ij}$. The noisy state is available when integrating \eqref{eq::reverse-sde}, so the only missing ingredient is $W_0^{ij}$, which is precisely what the generative model is trying to recover. We therefore reduce the marginal-score problem to estimating $W_0^{ij}$ and plugging the estimate into \eqref{eq::cond-score}.

\subsection{Learning by denoising}
\label{sec::learning}

We aim to estimate $W_0^{ij}$ from the noisy state $\tilde W_t^{ij}$. A direct approach would require samples of the unobserved graphon entries, which we do not have. The observed quantity at training time is the graph $A_0$, with $A_0^{ij}\sim\mathrm{Ber}(W_0^{ij})$ by \eqref{eq::ic-substitution}. Since $\mathbb{E}[A_0^{ij}\mid W_0^{ij}]=W_0^{ij}$, the binary entry is an unbiased proxy for the graphon entry, and we use the optimal predictor of $A_0^{ij}$ from the noisy state as our estimator of $W_0^{ij}$ when computing the marginal score from \eqref{eq::cond-score}.

We parametrize a graph neural network $\Phi_\theta(\tilde W_t, t)$ that, given the full noisy graph $\tilde W_t=(\tilde W_t^{ij})_{i,j=1}^N$ and the diffusion time, returns a per-edge probability $\hat A_{t,\theta}^{ij}\in[0,1]$. We follow the denoiser parameterization of \cite{qin2025defog}, which builds on the graph transformer of \cite{vignac2023digress} augmented with relative random walk probability features~\cite{ma2023grit, siraudin2024cometh}. We describe the architecture in Appendix~\ref{app::parameterization}. 
Since the goal is to denoise $W_t$ and recover the initial adjacency matrix $A_0$, a simple choice is an edge-wise binary denoising loss
\begin{equation}\label{eq::dsm-loss}
    \mathcal{L}(\theta)
    \;=\;
    -\,\mathbb{E}_{t,A_0,\tilde W_t}\!
    \biggl[\,
        \sum_{i,j=1}^{N}
        \Bigl(
            A_0^{ij}\,\log \Phi_\theta(\tilde W_t,t)^{ij}
            \;+\;
            \bigl(1-A_0^{ij}\bigr)\,\log\bigl(1-\Phi_\theta(\tilde W_t,t)^{ij}\bigr)
        \Bigr)
    \biggr],
\end{equation}
where $t$ is sampled uniformly on $[0,T]$, $A_0$ from the data, and $\tilde W_t$ from the simulated forward.
Noisy samples $\tilde W_t$ are produced by numerical integration of \eqref{eq::cell-sde-renorm} starting from $A_0$.
The full training loop is summarized in Algorithm~\ref{alg::training} of Appendix~\ref{app::algorithms}.

Once $\Phi_\theta$ has been trained, sampling proceeds in three steps. We initialize $\tilde W_T^{ij}$ independently from $\mathrm{Beta}(\mu)$ at every cell. At every reverse step we query $\hat A_0 = \Phi_\theta(\tilde W_t, t)$, plug $\hat A_0^{ij}$ into the conditional-score formula \eqref{eq::cond-score} as a surrogate for $W_0^{ij}$, and take a discretized step on the reverse-time SDE \eqref{eq::reverse-sde}. After reaching $t=0$, we draw a binary graph from the final state via the Bernoulli step \eqref{eq::structured-sampling}. The full procedure is summarized in Algorithm~\ref{alg::sampling} of Appendix~\ref{app::algorithms}.

%% file: sections/section5.tex
\section{Experiments}
\label{sec::experiments}

In this section we empirically validate the moment guarantees of Section~\ref{sec::graphon_to_graph} on the renormalized process \eqref{eq::cell-sde-renorm}, and the scalability properties of the DiPhon method. Extended experiments can be found in Appendix~\ref{app::experiments}.

\subsection{Forward moment alignment}
\label{sec::exp::moments}

This section validates the forward moment guarantees behind the discretization. Theorem~\ref{thm::moment-matching} shows that the renormalized cellwise SDE \eqref{eq::cell-sde-renorm} matches the first moment of the continuous spatial average exactly, while its second moment differs only through the in-cell variance term $\beta_0^{ij}$. Corollary~\ref{cor::forward-graph-ic} extends the prediction to graph initialization, $A_0^{ij}\sim\mathrm{Ber}(\bar W_0^{ij})$, where an additional variance gap appears solely because the initial condition is now random.

\textbf{Setup.} 
We test this forward moment alignment on a two-community SBM with equal communities, intra-edge probability $p=0.7$, inter-edge probability $q=0.2$, and stationary mean $\mu=(p+q)/2=0.45$. Since $N$ is even, the grid aligns with the community boundary at $x=1/2$. Hence each cell is constant, $\beta_0^{ij}=0$, and the theoretical second moments of the continuous spatial average and the renormalized discrete process coincide. We set $\kappa=1$, $\sigma=0.005$, $T=4$, and simulate \eqref{eq::cell-sde-renorm} with Euler--Maruyama using $n_{\text{steps}}=400$ and $M=1000$ Monte Carlo trajectories for different sizes $N$.

\textbf{Results.} Figure~\ref{fig::moments-main} reports the intra-community cell, $\bar W_0^{ij}=0.7$, for $N=20$ and $N=100$. The empirical mean follows $\mu+(\bar W_0^{ij}-\mu)e^{-\kappa t}$ in both deterministic and Bernoulli initializations, confirming the exact first-moment match. With deterministic initialization, the empirical variance follows the common continuous/discrete closed-form curve, as expected from $\beta_0^{ij}=0$. With Bernoulli initialization, the second moment differs at early times because the graph draw adds initial variance $\bar W_0^{ij}(1-\bar W_0^{ij})$. The observed gap then evolves according to Corollary~\ref{cor::forward-graph-ic}. The same behavior holds across resolutions and for the inter-community cell, reported in Appendix~\ref{app::experiments}.

\begin{figure}[t]
\centering
\includegraphics[width=\linewidth]{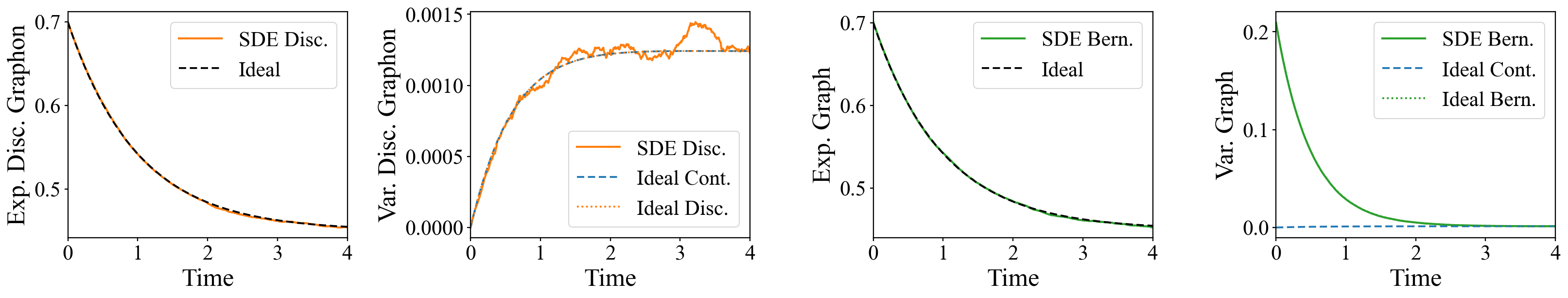}\\[2pt]
\includegraphics[width=\linewidth]{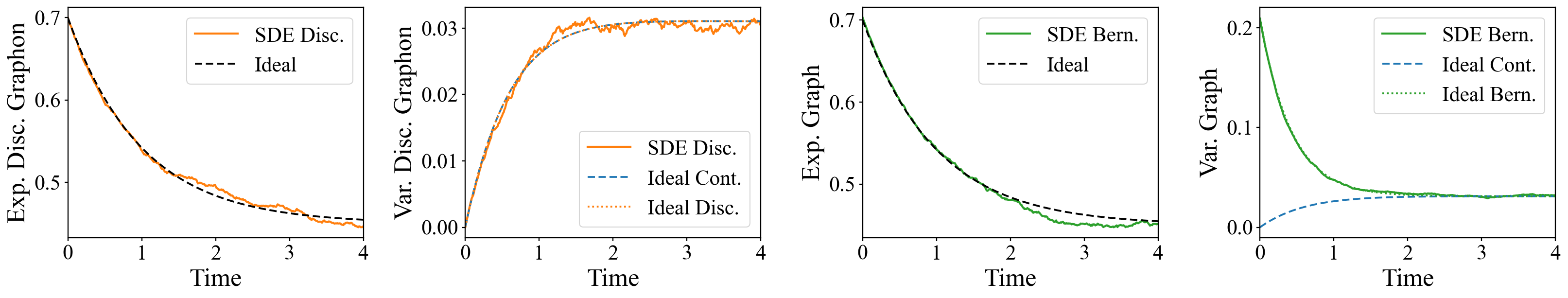}
\caption{SBM intra-community cell, $\sigma=0.005$, $T=4$, $M=10^{3}$. Top: $N=20$. Bottom: $N=100$. Empirical moments of \eqref{eq::cell-sde-renorm} (solid) match the closed-form references (dashed/dotted) under both deterministic and Bernoulli initial conditions.}
\label{fig::moments-main}
\end{figure}

\vspace{-8pt}
\subsection{Scalability of DiPhon}
\vspace{-4pt}

We evaluate whether DiPhon can generate graphs at sizes outside the training distribution. This is the key test for a graphon-inspired model: a model trained at one resolution should preserve structure when sampled at a larger one.

\vspace{-4pt}
\paragraph{Setup} We use three synthetic graph families: stochastic block models (SBMs), preferential attachment (PA) graphs, and trees, with SBMs and trees corresponding to graph types commonly used in graph generation benchmarks~\cite{qin2025defog}. SBMs are the most aligned with the graphon setting, since they admit a natural piecewise-constant graphon \cite{ruiz2020graphon}. PA graphs have not directly an associated graphon, but their line graph do \cite{kandanaarachchi2024graphons}. Trees are sparse and fall outside the classical graphon regime. They therefore test whether the proposed mechanism remains useful beyond the setting where the theory directly applies.
For DiPhon, we use a constant-parameter Jacobi diffusion. This is the simplest implementation of the proposed forward process: it acts pointwise on edge probabilities, stays bounded in $[0,1]$, and still preserves the first-moment matching property of the graphon-level dynamics. We use this parametrization to avoid adding extra modeling complexity.
All methods are trained on small graphs and tested on larger graphs. For instance, in the PA dataset we train on graphs with $40$ to $80$ nodes, while we evaluate generation up to $300$ nodes. We compare DiPhon with discrete diffusion methods, DiGress \cite{vignac2023digress} and DeFoG \cite{qin2025defog}, and Gaussian diffusion models, GDSS \cite{jo2022gdss} and GruM \cite{jo2024grum}. Although these methods are not typically used for out-of-scale generation, they can be applied to this setting directly, without technical modifications.
For each graph size, we report the fraction of generated graphs satisfying the defining property of the target family: we use forest accuracy for Tree graphs, which measures the fraction of generated graphs that are acyclic; for SBM graphs, we use a fitted-block-model statistical validity test~\cite{martinkus2022spectre}; for PA graphs, we use a validity check based on power-law-fit and structural criteria, including whether the fitted degree distribution is consistent with a preferential-attachment-like graph.

\vspace{-4pt}
\paragraph{Results} The results can be found in Figure~\ref{fig::ood_accuracy}. 
The results show that most methods extrapolate slightly beyond the training sizes, but only DiPhon remains stable across all datasets. 
In particular, DiPhon keeps high accuracy for SBMs, PA graphs, and trees, including sizes far beyond those seen during training. The tree results are particularly relevant, since they show that the method remains robust even when the graphon assumptions do not formally carry over.
Regarding the baselines, the discrete diffusion methods perform well near the training regime for all datasets, but their accuracy collapses at larger sizes. 
Gaussian diffusion models can be more stable at larger sizes than discrete diffusion methods, as illustrated by GruM on SBM graphs, but they remain less consistent than discrete diffusion models across graph families, with particularly poor performance on tree graphs.
Overall, DiPhon gives the strongest out-of-distribution size generalization among the tested methods. It performs well in the graphon-compatible SBM setting, remains accurate on PA graphs, and preserves structure on sparse trees. Additional results, setup explanations and visualizations of generated graphs are reported in Appendix~\ref{app::experiments}.

\begin{figure}[t]
    \centering
    \includegraphics[width=0.99\linewidth]{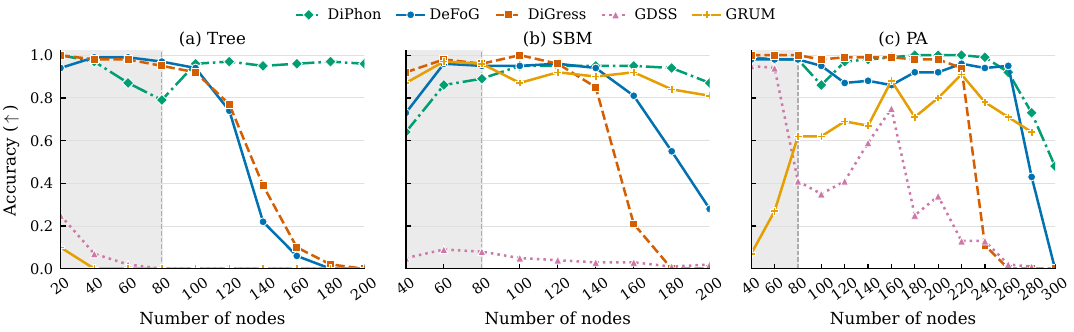}
    \caption{
    Out-of-distribution size generalization. DiPhon maintains high accuracy across number of nodes for all datasets, while discrete diffusion methods collapse at larger sizes and Gaussian baselines are less consistent.
    }
    \label{fig::ood_accuracy}
\end{figure}

%% file: sections/section6.tex
\section{Discussion and limitations}
\label{sec::limitations}

This work connects diffusion-based graph generative modeling with graphon theory by defining the forward dynamics directly on graphon space. This yields a scale-free diffusion process whose discretizations can be sampled at different graph sizes. DiPhon uses Jacobi dynamics to keep edge probabilities in the unit interval, approximately matches the first two moments of the continuous process at the discrete level, and recovers convergence in homomorphism density toward an Erd\H{o}s--Renyi reference. Empirically, this formulation enables models trained on small graphs to transfer their sampling distribution to larger graphs without retraining. The approach nevertheless has different limitations. The theoretical analysis is developed in the dense regime, where graphons are the canonical limit objects, so the formal guarantees do not fully explain the empirical behavior observed on sparse graph families. Moreover, the equivalence between the graphon process and its discrete counterpart is established only at the level of the first two moments. On the other hand, we do not bound the full distance between the marginal distributions, which may be more relevant for generative sampling. Finally, training requires simulating forward trajectories because the renormalized Jacobi transition does not admit a tractable closed-form density for direct sampling at arbitrary times, introducing additional overhead compared with Gaussian or fully discrete diffusions. Further modeling extensions, including node features, categorical edges, and weighted graphs, are natural within the framework but fall outside the binary-edge graphon guarantees developed here. A detailed discussion of these limitations is provided in \ref{app::lims}.

%% file: sections/section7.tex
\section{Conclusion}
\label{sec::conclusion}

We studied the scalability of graph generative modeling through the lens of graphons by defining the forward dynamics in graphon space and deriving a graph-level diffusion process through discretization. The proposed construction relies on a Jacobi SDE that is bounded by construction and admits a Beta stationary law. By taking the observed graph as the initial condition and sampling edges within each cell at every time step, we obtain DiPhon, a graph-valued diffusion process whose first moment exactly matches that of the discretized graphon process and whose second moment incurs a closed-form discrepancy that vanishes over time. DiPhon induces graph-valued trajectories from which samples can be generated, and, in the stationary regime, the sampled graphs converge in homomorphism density to the Erd\H{o}s--R\'enyi graphon for every motif. The reverse process is tractable and reduces training to a binary denoising objective on the observed graph.
Synthetic experiments confirm the predicted moment alignment across discretization sizes and show that DiPhon can be trained on small graphs while transferring its sampling distribution to substantially larger graphs without retraining. Together, the moment guarantees, the homomorphism-density limit, and the empirical size transfer support the use of graphons as the carrier of diffusion dynamics for size-agnostic graph generation. Natural directions for future work include training on large graph datasets with increasing sequences of graph sizes, extending the formal analysis beyond the dense regime, and exploring non-trivial graphon references that encode structural priors specific to a target domain.

%% file: sections/appendix_related_work.tex
\section{Extended related work}
\label{app::related-work}

This appendix expands the related-work discussion sketched in Section~\ref{sec::related-work}. We organize it along the two lines that frame our construction, graphons and diffusion processes, and we close with the classical references that underpin our use of the Jacobi diffusion and of the score-based generative paradigm.

\subsection{Graphons as size-agnostic descriptors of dense graphs}
\label{app::related-graphons}

Graphons arose as the limit object of sequences of dense graphs of growing size and are characterized by convergence in homomorphism density and in cut distance \cite{lovasz2006limits, borgs2008convergent, borgs2012convergent, lovasz2012large}. They generalize classical random graph models such as Erd\H{o}s--R\'enyi and the stochastic blockmodel \cite{holland1983sbm}, both of which arise as step graphons. Beyond pure combinatorics, graphons have served as a statistical estimation target for blockmodel and nonparametric graph models \cite{airoldi2013stochastic}, and they provide a representation of an underlying random-graph law that is invariant to the number of nodes. This invariance is what makes the graphon a natural object for reasoning about graph quantities across scales.

The most quantitative use of graphons as a transfer object so far has been the analysis of graph neural networks. \citet{ruiz2020graphon} introduced graphon neural networks as the limit of GNNs on graphs sampled from a common graphon and proved a transferability bound that decays with the number of nodes for bandlimited graph filters. The graphon-signal-processing formalism of \citet{ruiz2021graphonsp} underpins this analysis, and the spectral-filter transferability results of \citet{levie2021transferability} and \citet{maskey2023transferability} extend it to broader filter classes and to unbounded graphons. Together they establish the graphon as a quantitatively useful object for transferring architectures across graph sizes, but the theory addresses representation rather than generation.

Graphons have been comparatively less explored in the context of generative modeling. One line of work treats the graphon itself as the generative model, with estimation of the graphon as the central goal. \citet{xu2021graphon} learn a graphon-valued representation of a graph distribution and decode by sampling a graph from the resulting graphon, while \citet{xia2023implicit} parametrize a graphon as a continuous neural function so that graphs of arbitrary size can be obtained by re-sampling the input domain, an approach further refined by \citet{azizpour2024scalable}. Within diffusion models and related paradigms, the use of graphons is almost non-existent. The only exception is \citet{wijesinghe2026flowette}, who extend graphons into a probabilistic family called \emph{graphettes} where sparse motifs can be introduced, and then use specific graphettes as a prior sampling distribution that is transported to the data distribution via flow matching. Even there, the graphette acts as a structural prior on the source distribution rather than as the state space of the dynamics, so the generative process is run on graph-level representations and not on the graphon itself.

\subsection{Diffusion processes for graph generation}
\label{app::related-diffusion-graphs}

Graphs \cite{niu2020edpgnn, jo2022gdss, vignac2023digress}, together with information defined over them \cite{wen2023diffstg, uslu2026graph, rozada2026grph}, are objects of significant interest for generative modeling.
Diffusion-based graph generative models split along the type of state space they assume on the edges. 
Discrete approaches treat the edge variables as categorical and noise the graph through a continuous-time Markov chain on the discrete edge set, with denoising performed by a graph neural network conditioned on the current noisy categorical graph \cite{austin2021d3pm, haefeli2022discrete, vignac2023digress, qin2025defog, lou2024sedd}. 
The formulation aligns with the combinatorial nature of the data, but it does not extend to a continuous spatial domain in the sense required by a graphon. 
The edge state lives on a finite alphabet at each cell, with no notion of an edge probability evolving as a continuous random field, so the connection to the graphon limit is absent by construction.

The complementary line of work uses Gaussian SDEs on the adjacency entries, with score-based denoising in continuous space \cite{niu2020edpgnn, jo2022gdss, jo2024grum}. 
These models do operate pointwise on a continuous edge variable and are therefore geometrically compatible with the idea of a graphon-valued process, but their state is unbounded by construction. 
The Gaussian forward does not preserve the $[0,1]$ range that a graphon-valued field requires, and projection or thresholding at sampling time is needed to recover an edge probability. 
Their connection to a graphon limit is thus implicit at best, and a quantitative match of moments to a continuous-space graphon process is not available.

A natural way to combine both desiderata is to run the diffusion on a bounded domain that already matches the support of the data. 
On the probability simplex, \citet{avdeyev2023dirichlet} and \citet{stark2024dirichlet} construct continuous-time stochastic processes whose stationary law is Dirichlet, and \citet{eijkelboom2024catflow} adapt the same idea to flow models on simplex-valued categorical data with applications to graph generation. 
Diffusion on more general constrained or geometric domains has been studied via barrier-driven processes \cite{fishman2023constrained} and via score-based generative modeling on Riemannian manifolds \cite{debortoli2022riemannian, yim2023se3}. 
These constructions are all candidates for a graph generative process whose state respects the bounded range of edge probabilities, yet their use in graph generation has so far been limited and none of them is derived from a graphon-space process.

\subsection{Foundations of score-based diffusion and of the Jacobi process}
\label{app::related-foundations}

The forward and reverse SDEs that we invoke in Section~\ref{sec::reverse} rest on the time-reversal result of \citet{anderson1982reverse}, which characterizes the reverse-time dynamics of a diffusion in terms of the same diffusion coefficient and a drift corrected by the score of the marginal density. 
Its use as a generative principle was first proposed by \citet{sohldickstein2015deep}, and the modern score-based formulation was developed in \citet{song2019generative} and \citet{ho2020ddpm}, with the unified continuous-time framing in terms of forward and reverse SDEs given by \citet{song2021scoresde}. 
Score-based training relies on the score-matching estimator of \citet{hyvarinen2005score} and on the denoising-score-matching identity of \citet{vincent2011denoising}, which links the optimal denoiser of a noised sample to the score of the perturbed marginal.

The Jacobi diffusion that we adopt is a classical object. 
Its origin is the Wright--Fisher diffusion of population genetics \cite{wright1931evolution}, where it describes the evolution of allele frequencies on $[0,1]$, and it is treated as a textbook example of a one-dimensional diffusion on a bounded interval with Beta stationary law in \citet{karlin1981second} and \citet{etheridge2011mathematical}. 
The Pearson family of \citet{forman2008pearson} and the moment construction of \citet{bibby2005diffusion} place the Jacobi diffusion within a broader class of diffusions with prescribed stationary marginals and quadratic squared diffusion coefficient, while \citet{demni2009large} and \citet{gourieroux2006multivariate} extend the analysis to the asymptotic regime and to the multivariate Dirichlet-stationary case.

The spectral structure that we use to expose the closed form of the score in Section~\ref{sec::reverse} relies on the eigenfunction theory of the Jacobi diffusion. 
The analytic theory of the Jacobi polynomials themselves traces back to \citet{szego1939orthogonal}, and the modern semigroup treatment that frames the Jacobi diffusion as a canonical example of a Markov diffusion operator is given in \citet{bakry2014analysis}.

%% file: sections/appendix_jacobi_process.tex
\section{The one-dimensional Jacobi diffusion}
\label{app::jacobi}

This appendix collects the standard analytic properties of the scalar Jacobi diffusion on $[0,1]$ that are used throughout the paper: the forward SDE and the role of its parameters, the Beta stationary law, the closed-form marginal density obtained via eigenexpansion in shifted Jacobi polynomials, the closed-form expressions of the first two moments at every $t\ge 0$ and their asymptotic limits, and the closed-form expression of the score $\nabla_x\log p_t$. The renormalized process \eqref{eq::cell-sde-renorm} is a one-dimensional Jacobi diffusion at every cell, so the formulas of this appendix transfer directly to it after the parameter substitution \eqref{eq::renorm-params}.

\subsection{Forward process}
\label{app::jacobi-forward}

Fix three real parameters $\kappa>0$, $\sigma>0$, and $\mu\in(0,1)$. The scalar Jacobi diffusion on $[0,1]$ is the It\^o process
\begin{equation}\label{eq::app-jacobi-sde}
    dX_t
    \;=\;
    \kappa\bigl(\mu-X_t\bigr)\,dt
    \;+\;
    \sigma\,\sqrt{X_t\bigl(1-X_t\bigr)}\,dB_t,
    \qquad
    X_0\;=\;x_0\in[0,1],
\end{equation}
where $B_t$ is a standard Brownian motion. Each parameter has a distinct dynamical role. The drift $\kappa(\mu-X_t)$ is linear in the state, points toward $\mu$ from either side, and pulls $X_t$ toward $\mu$ at rate $\kappa$, so $\kappa$ controls how fast the mean trajectory equilibrates and $\mu$ sets the long-run target of the dynamics. The diffusion coefficient $\sigma\sqrt{X_t(1-X_t)}$ vanishes at the endpoints $\{0,1\}$ and is maximal at $X_t=\tfrac{1}{2}$, so the noise is state-dependent and turns off at the boundaries. This is what keeps $X_t$ confined to $[0,1]$ without reflection or projection. The amplitude $\sigma$ scales the noise globally and, together with $\kappa$, governs the dispersion of the stationary law: for fixed $\mu$, increasing $\sigma$ widens the dispersion around $\mu$, while increasing $\kappa$ concentrates it around $\mu$. Existence and uniqueness of strong solutions to \eqref{eq::app-jacobi-sde}, together with the absence of absorption at $\{0,1\}$ under these parameter ranges, are classical~\cite{karlin1981second, etheridge2011mathematical}.

\subsection{Stationary distribution}
\label{app::jacobi-stationary}

Setting
\begin{equation}\label{eq::app-beta-params}
    \alpha
    \;:=\;
    \frac{2\kappa\mu}{\sigma^2},
    \qquad
    \beta
    \;:=\;
    \frac{2\kappa(1-\mu)}{\sigma^2},
\end{equation}
the diffusion \eqref{eq::app-jacobi-sde} admits a unique stationary distribution on $[0,1]$, the Beta density
\begin{equation}\label{eq::app-beta-density}
    \pi(x)
    \;=\;
    \frac{x^{\alpha-1}(1-x)^{\beta-1}}{B(\alpha,\beta)},
    \qquad
    x\in(0,1),
\end{equation}
where $B(\alpha,\beta)$ is the Beta function~\cite{karlin1981second}. The associated law is denoted $\mathrm{Beta}(\mu)$ in the main text, parameterized by its mean $\mu$ rather than by $(\alpha,\beta)$. By construction $\alpha+\beta=2\kappa/\sigma^2$ and $\alpha/(\alpha+\beta)=\mu$, so the stationary mean is $\mu$ and the stationary variance is
\begin{equation}\label{eq::app-beta-stationary-var}
    \mathrm{Var}_\pi(X)
    \;=\;
    \frac{\alpha\beta}{(\alpha+\beta)^2(\alpha+\beta+1)}
    \;=\;
    \frac{\sigma^2\,\mu(1-\mu)}{2\kappa+\sigma^2}.
\end{equation}
The boundary behavior of $\pi$ depends on $\alpha$ and $\beta$: $\pi$ is integrable for all $\alpha,\beta>0$, bounded on $[0,1]$ when both exceed $1$, and divergent at the corresponding endpoint when either falls below $1$.

\subsection{Marginal density and eigenexpansion}
\label{app::jacobi-marginal}

For every $t>0$ the transition law of \eqref{eq::app-jacobi-sde} admits a density $p_t(\cdot\,|\,x_0)$ with respect to Lebesgue measure on $[0,1]$, the marginal $p_t$ of $X_t$ given $X_0=x_0$. The Jacobi diffusion is symmetric in $L^2([0,1],\pi)$ and its dynamics admit a discrete spectral expansion in the shifted Jacobi polynomials orthonormalized with respect to $\pi$, which we denote by $\phi_n$ for $n\ge 0$~\cite{karlin1981second, bakry2014analysis}. The polynomial $\phi_n$ has degree $n$, is computable in closed form from the standard Jacobi recurrence~\cite{szego1939orthogonal}, and satisfies the orthonormality relation
\begin{equation}\label{eq::app-jacobi-orthogonality}
    \int_0^1 \phi_n(x)\,\phi_m(x)\,\pi(x)\,dx
    \;=\;
    \delta_{n,m}.
\end{equation}
Mode $n$ decays at the rate
\begin{equation}\label{eq::app-jacobi-eigenvalue}
    \lambda_n
    \;:=\;
    n\!\left[\kappa+\tfrac{(n-1)\sigma^2}{2}\right],
    \qquad n\ge 0.
\end{equation}
The first three rates read $\lambda_0=0$, $\lambda_1=\kappa$, and $\lambda_2=2\kappa+\sigma^2$, which are exactly the rates governing the mean and variance trajectories below.

The transition density admits the spectral expansion
\begin{equation}\label{eq::app-transition-density}
    p_t(x\,|\,x_0)
    \;=\;
    \pi(x)\,\sum_{n=0}^{\infty}\,
    e^{-\lambda_n t}\,\phi_n(x_0)\,\phi_n(x),
    \qquad x\in(0,1),\;t>0,
\end{equation}
which converges in $L^2([0,1],\pi)$ uniformly on $t\ge t_0$ for every $t_0>0$~\cite{bakry2014analysis}. The expansion has a transparent interpretation. The constant mode $n=0$ contributes $\pi(x)$ and is the only mode that survives at infinity: for every $x_0$, $p_t(\,\cdot\,|\,x_0)\to\pi$ as $t\to\infty$ in $L^2([0,1],\pi)$, recovering the Beta stationary law. Each higher mode $n\ge 1$ encodes the contribution of degree-$n$ deviations from $\pi$ in the initial condition and decays exponentially at the polynomial-in-$n$ rate $\lambda_n$, so higher-frequency components of the initial mass are smoothed out faster than lower-frequency ones. Equation \eqref{eq::app-transition-density} is the closed-form description of $p_t$ that we refer to as the eigenexpansion.

\subsection{First moment}
\label{app::jacobi-mean}

Taking expectations in \eqref{eq::app-jacobi-sde}, the martingale term vanishes and $\mathbb{E}[X_t]$ solves the linear ordinary differential equation (ODE)
\begin{equation}\label{eq::app-mean-ode}
    \frac{d}{dt}\mathbb{E}[X_t]
    \;=\;
    \kappa\bigl(\mu-\mathbb{E}[X_t]\bigr),
    \qquad
    \mathbb{E}[X_0]=x_0.
\end{equation}
Its unique solution is
\begin{equation}\label{eq::app-mean-closed}
    \mathbb{E}[X_t]
    \;=\;
    \mu+\bigl(x_0-\mu\bigr)\,e^{-\kappa t},
    \qquad t\ge 0.
\end{equation}
The mean trajectory is deterministic, decays exponentially toward $\mu$ at rate $\kappa$, is monotone in $t$, and is independent of $\sigma$. The decay rate matches the rate $\lambda_1=\kappa$ in \eqref{eq::app-jacobi-eigenvalue}: the first moment is the projection of $X_t$ onto the degree-one mode of the eigenexpansion.

\subsection{Second moment and variance}
\label{app::jacobi-variance}

Applying It\^o's formula to $X_t^2$, taking expectations, and subtracting $\frac{d}{dt}\mathbb{E}[X_t]^2$ from the resulting equation, the variance $v_t:=\mathrm{Var}(X_t)$ solves the linear ODE
\begin{equation}\label{eq::app-var-ode}
    \frac{d}{dt}v_t
    \;=\;
    -(2\kappa+\sigma^2)\,v_t
    \;+\;
    \sigma^2\,\mathbb{E}[X_t]\bigl(1-\mathbb{E}[X_t]\bigr),
    \qquad
    v_0=0,
\end{equation}
with damping rate $2\kappa+\sigma^2=\lambda_2$ matching the second-mode rate in \eqref{eq::app-jacobi-eigenvalue}. The forcing $\sigma^2\,\mathbb{E}[X_t](1-\mathbb{E}[X_t])$ is the expected diffusion intensity at the deterministic mean trajectory and admits the expansion
\begin{equation}\label{eq::app-var-forcing-expansion}
    \mathbb{E}[X_t]\bigl(1-\mathbb{E}[X_t]\bigr)
    \;=\;
    \mu(1-\mu)
    \;+\;
    (x_0-\mu)(1-2\mu)\,e^{-\kappa t}
    \;-\;
    (x_0-\mu)^2\,e^{-2\kappa t},
\end{equation}
into three exponential families with rates $\{0,\kappa,2\kappa\}$. Solving \eqref{eq::app-var-ode} by Duhamel with the integrating factor $e^{(2\kappa+\sigma^2)t}$ and using the three standard integrals
\begin{align}
    \int_0^t e^{-(2\kappa+\sigma^2)(t-u)}\,du
    &\;=\;
    \frac{1-e^{-(2\kappa+\sigma^2)t}}{2\kappa+\sigma^2},
    \\[2pt]
    \int_0^t e^{-(2\kappa+\sigma^2)(t-u)}\,e^{-\kappa u}\,du
    &\;=\;
    \frac{e^{-\kappa t}-e^{-(2\kappa+\sigma^2)t}}{\kappa+\sigma^2},
    \\[2pt]
    \int_0^t e^{-(2\kappa+\sigma^2)(t-u)}\,e^{-2\kappa u}\,du
    &\;=\;
    \frac{e^{-2\kappa t}-e^{-(2\kappa+\sigma^2)t}}{\sigma^2},
\end{align}
yields the closed-form variance
\begin{align}
    v_t
    &\;=\;
    \frac{\sigma^2\,\mu(1-\mu)}{2\kappa+\sigma^2}\bigl(1-e^{-(2\kappa+\sigma^2)t}\bigr)
    \nonumber\\
    &\quad\;+\;
    \frac{\sigma^2\,(x_0-\mu)(1-2\mu)}{\kappa+\sigma^2}\bigl(e^{-\kappa t}-e^{-(2\kappa+\sigma^2)t}\bigr)
    \nonumber\\
    &\quad\;-\;
    (x_0-\mu)^2\bigl(e^{-2\kappa t}-e^{-(2\kappa+\sigma^2)t}\bigr).
    \label{eq::app-var-closed}
\end{align}
The expression is exact for every $t\ge 0$, vanishes at $t=0$, and is non-negative throughout because the underlying process is real-valued.

\subsection{Asymptotic moments}
\label{app::jacobi-asymptotic}

Letting $t\to\infty$ in \eqref{eq::app-mean-closed} and \eqref{eq::app-var-closed}, every transient term that carries an exponential factor vanishes and only the rate-zero contribution survives. The asymptotic mean and variance are
\begin{equation}\label{eq::app-asymptotic-moments}
    \lim_{t\to\infty}\mathbb{E}[X_t]
    \;=\;
    \mu,
    \qquad
    \lim_{t\to\infty}\mathrm{Var}(X_t)
    \;=\;
    \frac{\sigma^2\,\mu(1-\mu)}{2\kappa+\sigma^2},
\end{equation}
which agree with the mean and variance of $\mathrm{Beta}(\mu)$ recorded in \eqref{eq::app-beta-stationary-var}. The asymptotic mean is independent of the noise amplitude $\sigma$: only $\mu$ enters. The asymptotic variance scales as $\sigma^2/(2\kappa+\sigma^2)$ in the parameters and as $\mu(1-\mu)$ in the target, so it is maximized at $\mu=\tfrac{1}{2}$ for fixed $(\kappa,\sigma)$ and vanishes as $\mu\to 0$ or $\mu\to 1$.

\subsection{Score function}
\label{app::jacobi-score}

The score is the spatial gradient of the log marginal density. For every $t>0$ and $x\in(0,1)$,
\begin{equation}\label{eq::app-score-def}
    s_t(x\,|\,x_0)
    \;:=\;
    \nabla_x\log p_t(x\,|\,x_0).
\end{equation}
The eigenexpansion \eqref{eq::app-transition-density} factors $p_t(x\,|\,x_0)$ as the product of $\pi(x)$ and the conditional series $\sum_{n=0}^{\infty}e^{-\lambda_n t}\,\phi_n(x_0)\,\phi_n(x)$, which is strictly positive on $(0,1)$ for every $t>0$ and $x_0\in[0,1]$. Differentiating the logarithm of the product splits the score additively into a stationary contribution and a transient contribution,
\begin{equation}\label{eq::app-score-split}
    s_t(x\,|\,x_0)
    \;=\;
    \nabla_x\log\pi(x)
    \;+\;
    \frac{\nabla_x\!\sum_{n=0}^{\infty} e^{-\lambda_n t}\,\phi_n(x_0)\,\phi_n(x)}
         {\sum_{n=0}^{\infty} e^{-\lambda_n t}\,\phi_n(x_0)\,\phi_n(x)}.
\end{equation}
The first contribution is the score of the Beta stationary law and is computed from \eqref{eq::app-beta-density},
\begin{equation}\label{eq::app-score-beta}
    \nabla_x\log\pi(x)
    \;=\;
    \frac{\alpha-1}{x}\;-\;\frac{\beta-1}{1-x}
    \;=\;
    \frac{2\kappa\mu-\sigma^2}{\sigma^2\,x}
    \;-\;
    \frac{2\kappa(1-\mu)-\sigma^2}{\sigma^2\,(1-x)},
\end{equation}
where the second equality substitutes \eqref{eq::app-beta-params}. The second contribution is the logarithmic derivative of the conditional series. The orthonormalized polynomials $\phi_n$ inherit the standard Jacobi derivative recurrence~\cite{szego1939orthogonal}, so $\phi_n'$ is again a polynomial of degree $n-1$ that is computable in closed form to any order. Substituting term-by-term into the derivative of the conditional series in \eqref{eq::app-score-split} and combining with \eqref{eq::app-score-beta} gives the closed-form expression
\begin{equation}\label{eq::app-score-closed}
    s_t(x\,|\,x_0)
    \;=\;
    \nabla_x\log\pi(x)
    \;+\;
    \frac{\displaystyle\sum_{n=0}^{\infty} e^{-\lambda_n t}\,\phi_n(x_0)\,\phi_n'(x)}
         {\displaystyle\sum_{n=0}^{\infty} e^{-\lambda_n t}\,\phi_n(x_0)\,\phi_n(x)},
\end{equation}
where the $n=0$ mode contributes zero to the numerator because $\phi_0$ is constant. Equation \eqref{eq::app-score-closed} is the formula instantiated as \eqref{eq::cond-score} in Section~\ref{sec::reverse}. The expression has two limits of interest. As $t\to\infty$, every transient term in numerator and denominator decays exponentially at rate $\lambda_n\ge\lambda_1=\kappa$ except for the $n=0$ denominator mode $\phi_0(x_0)\phi_0(x)$, which is constant in $t$. The ratio in \eqref{eq::app-score-closed} therefore tends to zero and the score collapses to the stationary score $\nabla_x\log\pi(x)$ of \eqref{eq::app-score-beta}, consistent with the convergence of $p_t$ to $\pi$. As $t\to 0^+$, the conditional series concentrates around $x=x_0$ and the transient contribution becomes singular at $x=x_0$, recovering the delta initial condition $p_0(\,\cdot\,|\,x_0)=\delta_{x_0}$.

%% file: sections/appendix_proof_moments_matching.tex
\section{Proof of the moment-matching theorem}
\label{app::moments-matching}

This appendix proves Theorem~\ref{thm::moment-matching} of Section~\ref{sec::discretization}. The argument is organized in three steps. Lemma~\ref{lem::moments-continuous} characterizes the first and second moments of the cell average $\bar W_t^{ij}$ obtained by spatially averaging the formal SDE \eqref{eq::spde} over the cell $I_i\times I_j$. Lemma~\ref{lem::moments-renormalized} characterizes the first and second moments of the renormalized scalar SDE $\tilde W_t^{ij}$ defined in \eqref{eq::cell-sde-renorm} with parameters \eqref{eq::renorm-params}. Theorem~\ref{thm::moment-matching} then follows by combining the two lemmas and computing the discrepancy in closed form.

Throughout this appendix we work on a single cell $I_i\times I_j\subset[0,1]^2$ of Lebesgue measure $|C|=1/N^2$, and we adopt the shorthand
\begin{equation}\label{eq::app-shorthand}
\bar m_t \;:=\; \mathbb{E}\bigl[\bar W_t^{ij}\bigr],
\qquad
\tilde m_t \;:=\; \mathbb{E}\bigl[\tilde W_t^{ij}\bigr],
\qquad
\bar v_t \;:=\; \mathrm{Var}\bigl(\bar W_t^{ij}\bigr),
\qquad
\tilde v_t \;:=\; \mathrm{Var}\bigl(\tilde W_t^{ij}\bigr).
\end{equation}
The cell index $(i,j)$ is omitted whenever no ambiguity arises.

\subsection{First and second moment of the spatial average}
\label{app::lemma-continuous}

\begin{lemma}[Moments of the spatial average]
\label{lem::moments-continuous}
Let $\bar W_t^{ij}$ be the cell average defined in \eqref{eq::cell-average}, with deterministic initial condition $W_0$, and set $\bar m_0:=\mathbb{E}[\bar W_0^{ij}] = N^2\!\int_{I_i\times I_j} W_0(x,y)\,dx\,dy$. Let
\begin{equation}\label{eq::beta0-def}
    \beta_0^{ij}
    \;:=\;
    N^2\!\!\int_{I_i\times I_j}\!\!\bigl(W_0(x,y)-\bar m_0\bigr)^2\,dx\,dy
\end{equation}
denote the in-cell variance of the initial graphon. Then for every $t\ge 0$,
\begin{align}
\bar m_t
&\;=\;
\mu + (\bar m_0-\mu)\,e^{-\kappa t}, \label{eq::lem1-mean}
\\[4pt]
\bar v_t
&\;=\;
N^2 \biggl[\,
\frac{\sigma^2 \mu(1-\mu)}{2\kappa+\sigma^2}\bigl(1-e^{-(2\kappa+\sigma^2)t}\bigr)
\nonumber\\
&\qquad\;\;\;
\;+\;
\frac{\sigma^2 (\bar m_0-\mu)(1-2\mu)}{\kappa+\sigma^2}\bigl(e^{-\kappa t}-e^{-(2\kappa+\sigma^2)t}\bigr)
\nonumber\\
&\qquad\;\;\;
\;-\;
\bigl((\bar m_0-\mu)^2+\beta_0^{ij}\bigr)\bigl(e^{-2\kappa t}-e^{-(2\kappa+\sigma^2)t}\bigr)
\,\biggr]. \label{eq::lem1-var}
\end{align}
\end{lemma}

\begin{proof}
We derive the second moment of $\bar W_t^{ij}$ at the level of spatial averages by introducing an auxiliary spatially averaged squared field, which avoids resolving $W_t(x,y)^2$ pointwise.

Taking expectations pointwise in \eqref{eq::spde}, the martingale term vanishes and $\mathbb{E}[W_t(x,y)]$ satisfies
\begin{equation}\label{eq::lem1-pw-mean-ode}
    \frac{d}{dt}\mathbb{E}[W_t(x,y)]
    \;=\;
    \kappa\bigl(\mu - \mathbb{E}[W_t(x,y)]\bigr).
\end{equation}
Integrating \eqref{eq::lem1-pw-mean-ode} over $I_i\times I_j$, normalizing by $|C|$, and exchanging expectation with the spatial integral, the first moment $\bar m_t$ of the cell average satisfies the linear scalar ODE $\frac{d}{dt}\bar m_t = \kappa(\mu - \bar m_t)$ with initial condition $\bar m_0$. Its unique solution is \eqref{eq::lem1-mean}.

We now turn to the second moment. The derivation proceeds in three steps. First, we introduce an auxiliary spatially integrated second moment $r_t$ and shows that it satisfies a closed linear ODE driven by $\bar m_t$. Then, we write an ODE for $\bar v_t$ via the identity $\bar v_t = q_t - \bar m_t^2$, where $q_t := \mathbb{E}[(\bar W_t^{ij})^2]$, and finds that this ODE depends on $r_t$. Lastly, we eliminate $r_t$ through an algebraic identity, yielding a self-contained ODE for $\bar v_t$ that we then solve by Duhamel.

\textbf{ODE for the auxiliary second moment $r_t$.}
We define the auxiliary quantity
\begin{equation}\label{eq::r-def}
    r_t
    \;:=\;
    \mathbb{E}\!\left[\,N^2\!\!\int_{I_i\times I_j}\!\! W_t(x,y)^2\,dx\,dy\,\right].
\end{equation}
The motivation is that, as we will see below, both the drift and the quadratic variation produced by It\^o's formula on $W_t(x,y)^2$ reduce, after spatial integration over the cell, to expressions linear in $\bar m_t$ and $r_t$. This makes $r_t$ satisfy a closed linear ODE driven by $\bar m_t$, which we subsequently couple to the variance through a single term to be eliminated in the last step. Applying It\^o's formula to $W_t(x,y)^2$ inside the spatial integral,
\begin{equation}\label{eq::dY}
    d\!\left(N^2\!\!\int_{I_i\times I_j}\!\!W_t(x,y)^2\,dx\,dy\right)
    \;=\;
    N^2\!\!\int_{I_i\times I_j}\!\!\Bigl[\,\underbrace{2W_t(x,y)\,dW_t(x,y)}_{\text{drift contribution}} \;+\; \underbrace{dW_t(x,y)\,dW_t(x,y)}_{\text{quadratic variation}}\,\Bigr]\,dx\,dy.
\end{equation}
We treat the two underbraced terms separately.

For the drift contribution, multiplying the drift of \eqref{eq::spde} by $2W_t(x,y)$ and taking expectations,
\begin{align}\label{eq::r-drift}
    \mathbb{E}\!\left[N^2\!\!\int_{I_i\times I_j}\!\!2W_t(x,y)\,dW_t(x,y)\,dx\,dy\right]
    &\;=\;
    \mathbb{E}\!\left[N^2\!\!\int_{I_i\times I_j}\!\!2\kappa\,W_t(x,y)\bigl(\mu - W_t(x,y)\bigr)\,dx\,dy\right]\,dt
    \nonumber\\
    &\;=\;
    2\kappa\mu\,\mathbb{E}\!\left[N^2\!\!\int_{I_i\times I_j}\!\!W_t(x,y)\,dx\,dy\right]\,dt
    \nonumber\\
    &\quad
    \;-\;
    2\kappa\,\mathbb{E}\!\left[N^2\!\!\int_{I_i\times I_j}\!\!W_t(x,y)^2\,dx\,dy\right]\,dt
    \nonumber\\
    &\;=\;
    2\kappa\,\bigl(\mu \bar m_t - r_t\bigr)\,dt,
\end{align}
where the It\^o stochastic integral has zero mean, and the last equality recognizes $N^2\,\mathbb{E}\!\int_{I_i\times I_j}\!W_t\,dx\,dy = \bar m_t$ and $N^2\,\mathbb{E}\!\int_{I_i\times I_j}\!W_t^2\,dx\,dy = r_t$ by definition.

For the quadratic variation, recall that $dt^2=0$ and $dt\,dB_t=0$, so only the noise-times-noise term remains. Because $B_t$ is a space-time white noise, its formal It\^o product is $dB_t(x,y)\,dB_t(x',y')=\delta(x-x')\delta(y-y')\,dt$, so the bilinear form pairs the noise term in \eqref{eq::spde} with itself only on the spatial diagonal $(x',y')=(x,y)$. Substituting the diffusion coefficient $\sigma\sqrt{W_t(1-W_t)}$ and integrating one of the deltas against the spatial measure,
\begin{align}\label{eq::r-qv-pre}
    &N^2\!\!\int_{I_i\times I_j}\!\!dW_t(x,y)\,dW_t(x,y)\,dx\,dy \nonumber\\
    &\;=\;
    N^2\!\!\int_{I_i\times I_j}\!\!\int_{I_i\times I_j}\!\!\sigma^2\sqrt{W_t(x,y)\bigl(1-W_t(x,y)\bigr)}\,\sqrt{W_t(x',y')\bigl(1-W_t(x',y')\bigr)} \delta(x-x')\delta(y-y')\,dx\,dy\,dx'\,dy'\,dt
    \nonumber\\
    &\;=\;
    \sigma^2\,N^2\!\!\int_{I_i\times I_j}\!\!W_t(x,y)\bigl(1-W_t(x,y)\bigr)\,dx\,dy\,dt.
\end{align}
Taking expectations and using once more the definitions $N^2\,\mathbb{E}\!\int W_t\,dx\,dy = \bar m_t$ and $N^2\,\mathbb{E}\!\int W_t^2\,dx\,dy = r_t$,
\begin{equation}\label{eq::r-qv}
    \mathbb{E}\!\left[N^2\!\!\int_{I_i\times I_j}\!\!dW_t(x,y)\,dW_t(x,y)\,dx\,dy\right]
    \;=\;
    \sigma^2\,\bigl(\bar m_t - r_t\bigr)\,dt.
\end{equation}
Summing \eqref{eq::r-drift} and \eqref{eq::r-qv}, the auxiliary $r_t$ satisfies the linear ODE
\begin{equation}\label{eq::r-ode}
    \frac{d}{dt}r_t
    \;=\;
    -(2\kappa+\sigma^2)\,r_t
    \;+\;
    (2\kappa\mu+\sigma^2)\,\bar m_t,
    \qquad
    r_0 \;=\; \bar m_0^2 + \beta_0^{ij}.
\end{equation}
The initial condition follows from expanding $W_0(x,y) = \bar m_0 + (W_0(x,y)-\bar m_0)$ inside the spatial integral defining $r_0$,
\begin{align}
    r_0
    &\;=\;
    N^2\!\!\int_{I_i\times I_j}\!\!\bigl[\bar m_0 + (W_0(x,y)-\bar m_0)\bigr]^2 dx\,dy
    \nonumber\\
    &\;=\;
    \bar m_0^2
    \;+\;
    2\bar m_0\,N^2\!\!\int_{I_i\times I_j}\!\!\bigl(W_0(x,y)-\bar m_0\bigr)\,dx\,dy
    \nonumber\\
    &\;+\;
    N^2\!\!\int_{I_i\times I_j}\!\!\bigl(W_0(x,y)-\bar m_0\bigr)^2 dx\,dy
    \;=\;
    \bar m_0^2 + \beta_0^{ij},
\end{align}
where the cross-term vanishes because $N^2\!\int_{I_i\times I_j} W_0\,dx\,dy = \bar m_0$ by definition of $\bar m_0$, and the last term equals $\beta_0^{ij}$ by \eqref{eq::beta0-def}.

\textbf{ODE for the variance $\bar v_t$.}
The variance of the cell average decomposes as $\bar v_t = q_t - \bar m_t^2$ with $q_t := \mathbb{E}[(\bar W_t^{ij})^2]$. We first derive an ODE for $q_t$ and then convert it into an ODE for $\bar v_t$ via this identity. Integrating the SPDE \eqref{eq::spde} over $I_i\times I_j$ and applying the prefactor $N^2$, the cell average satisfies
\begin{equation}\label{eq::dW-bar}
    d\bar W_t^{ij}
    \;=\;
    \kappa\bigl(\mu - \bar W_t^{ij}\bigr)\,dt
    \;+\;
    N^2\!\!\int_{I_i\times I_j}\!\!\sigma\sqrt{W_t(x,y)\bigl(1-W_t(x,y)\bigr)}\,dB_t(x,y)\,dx\,dy.
\end{equation}
Apply It\^o's formula to $(\bar W_t^{ij})^2$,
\begin{equation}\label{eq::dq}
    d(\bar W_t^{ij})^2
    \;=\;
    \underbrace{2\,\bar W_t^{ij}\,d\bar W_t^{ij}}_{\text{drift contribution}}
    \;+\;
    \underbrace{(d\bar W_t^{ij})^2}_{\text{quadratic variation}}.
\end{equation}
For the drift contribution, multiplying \eqref{eq::dW-bar} by $2\bar W_t^{ij}$ and taking expectations,
\begin{align}\label{eq::q-drift}
    \mathbb{E}\bigl[\,2\,\bar W_t^{ij}\,d\bar W_t^{ij}\bigr]
    &\;=\;
    \mathbb{E}\bigl[\,2\,\bar W_t^{ij}\,\kappa(\mu - \bar W_t^{ij})\,\bigr]\,dt
    \nonumber\\
    &\;=\;
    2\kappa\mu\,\mathbb{E}\bigl[\bar W_t^{ij}\bigr]\,dt
    \;-\;
    2\kappa\,\mathbb{E}\bigl[(\bar W_t^{ij})^2\bigr]\,dt
    \nonumber\\
    &\;=\;
    2\kappa\,\bigl(\mu \bar m_t - q_t\bigr)\,dt,
\end{align}
where the It\^o stochastic integral vanishes in expectation, and the last equality uses $\mathbb{E}[\bar W_t^{ij}] = \bar m_t$ together with $\mathbb{E}[(\bar W_t^{ij})^2] = q_t$ by definition.

For the quadratic variation $(d\bar W_t^{ij})^2$, only the noise--noise contribution survives. Reproducing the diagonal collapse of the white-noise covariance as in \eqref{eq::r-qv-pre},
\begin{equation}\label{eq::q-qv-pre}
    (d\bar W_t^{ij})^2
    \;=\;
    \sigma^2\,N^4\!\!\int_{I_i\times I_j}\!\!W_t(x,y)\bigl(1-W_t(x,y)\bigr)\,dx\,dy\,dt.
\end{equation}
Taking expectations and using $N^2\mathbb{E}\!\int W_t\,dx\,dy = \bar m_t$ together with $N^2\mathbb{E}\!\int W_t^2\,dx\,dy = r_t$,
\begin{equation}\label{eq::q-qv}
    \mathbb{E}\bigl[(d\bar W_t^{ij})^2\bigr]
    \;=\;
    \sigma^2\,N^2\,\bigl(\bar m_t - r_t\bigr)\,dt.
\end{equation}
Summing \eqref{eq::q-drift} and \eqref{eq::q-qv} yields the linear ODE for $q_t$,
\begin{equation}\label{eq::q-ode}
    \frac{d}{dt}q_t
    \;=\;
    2\kappa\,\bigl(\mu \bar m_t - q_t\bigr)
    \;+\;
    \sigma^2\,N^2\,\bigl(\bar m_t - r_t\bigr).
\end{equation}

To convert \eqref{eq::q-ode} into an ODE for the variance, we use $\bar v_t = q_t - \bar m_t^2$ and compute $\frac{d}{dt}\bar m_t^2 = 2\bar m_t\,\kappa(\mu - \bar m_t) = 2\kappa\mu \bar m_t - 2\kappa \bar m_t^2$. Subtracting this from \eqref{eq::q-ode}, the cross-term $2\kappa\mu \bar m_t$ cancels and we obtain
\begin{equation}\label{eq::V-ode-raw}
    \frac{d}{dt}\bar v_t
    \;=\;
    -2\kappa\,\bar v_t
    \;+\;
    \sigma^2\,N^2\,\bigl(\bar m_t - r_t\bigr),
    \qquad
    \bar v_0 \;=\; 0,
\end{equation}
where $\bar v_0 = 0$ follows from the deterministic initial condition. Equation \eqref{eq::V-ode-raw} is not yet self-contained because of the explicit dependence on $r_t$, which we eliminate next through an algebraic identity.

\textbf{Integrating the resulting ODE.}
To close \eqref{eq::V-ode-raw}, we express $r_t$ in terms of $\bar m_t$ and $\bar v_t$ alone. Define the discrepancy
\begin{equation}\label{eq::D-def}
    \delta_t
    \;:=\;
    r_t \;-\; \bar m_t^2 \;-\; \bar v_t/N^2.
\end{equation}
Differentiating \eqref{eq::D-def} and substituting \eqref{eq::r-ode}, $\frac{d}{dt}\bar m_t^2 = 2\kappa\mu \bar m_t - 2\kappa \bar m_t^2$, and \eqref{eq::V-ode-raw},
\begin{align}
    \frac{d}{dt}\delta_t
    &\;=\;
    -(2\kappa+\sigma^2)\,r_t + (2\kappa\mu+\sigma^2)\,\bar m_t
    \;-\;
    2\kappa\mu \bar m_t + 2\kappa \bar m_t^2
    \;-\;
    \tfrac{1}{N^2}\bigl[\,-2\kappa\,\bar v_t + \sigma^2 N^2(\bar m_t - r_t)\,\bigr]
    \nonumber\\
    &\;=\;
    -2\kappa\bigl(r_t - \bar m_t^2 - \bar v_t/N^2\bigr)
    \;=\;
    -2\kappa\,\delta_t.
\end{align}
The initial condition is $\delta_0 = r_0 - \bar m_0^2 - 0 = \beta_0^{ij}$ by \eqref{eq::r-ode}, so $\delta_t = \beta_0^{ij}\,e^{-2\kappa t}$. Equivalently,
\begin{equation}\label{eq::r-from-V}
    r_t
    \;=\;
    \bar m_t^2 \;+\; \bar v_t/N^2 \;+\; \beta_0^{ij}\,e^{-2\kappa t}.
\end{equation}

Substitute \eqref{eq::r-from-V} into the right-hand side of \eqref{eq::V-ode-raw}. The cross-term $\sigma^2 N^2 \cdot \bar v_t/N^2 = \sigma^2 \bar v_t$ shifts the damping coefficient from $2\kappa$ to $2\kappa+\sigma^2$, and the resulting ODE is self-contained,
\begin{equation}\label{eq::V-ode-clean}
    \frac{d}{dt}\bar v_t
    \;=\;
    -(2\kappa+\sigma^2)\,\bar v_t
    \;+\;
    \sigma^2\,N^2\,\bar m_t(1-\bar m_t)
    \;-\;
    \sigma^2\,N^2\,\beta_0^{ij}\,e^{-2\kappa t},
    \qquad
    \bar v_0 \;=\; 0.
\end{equation}
Substituting $\bar m_u = \mu+(\bar m_0-\mu)e^{-\kappa u}$ into the nonlinear term $\bar m_u(1-\bar m_u)$ and expanding,
\begin{equation}\label{eq::m-times-1minusm}
    \bar m_u(1-\bar m_u)
    \;=\;
    \mu(1-\mu)
    \;+\;
    (\bar m_0-\mu)(1-2\mu)\,e^{-\kappa u}
    \;-\;
    (\bar m_0-\mu)^2\,e^{-2\kappa u}.
\end{equation}
The right-hand side of \eqref{eq::V-ode-clean} is therefore a sum of three exponential families with rates $\{0,\,\kappa,\,2\kappa\}$, where the coefficient at rate $2\kappa$ collects both the $-(\bar m_0-\mu)^2$ contribution from \eqref{eq::m-times-1minusm} and the $-\beta_0^{ij}$ contribution sourced by the auxiliary identity \eqref{eq::r-from-V}. We now solve \eqref{eq::V-ode-clean} by Duhamel with the integrating factor $e^{(2\kappa+\sigma^2)t}$ and substitute \eqref{eq::m-times-1minusm} into the integrand,
\begin{align}\label{eq::V-duhamel}
    \bar v_t
    &\;=\;
    \int_0^t e^{-(2\kappa+\sigma^2)(t-u)}
    \Bigl[\,\sigma^2 N^2\,\bar m_u(1-\bar m_u) \;-\; \sigma^2 N^2\,\beta_0^{ij}\,e^{-2\kappa u}\,\Bigr]\,du
    \nonumber\\
    &\;=\;
    \sigma^2 N^2 \int_0^t e^{-(2\kappa+\sigma^2)(t-u)}
    \Bigl[\,\mu(1-\mu) \;+\; (\bar m_0-\mu)(1-2\mu)\,e^{-\kappa u}
    \nonumber\\
    &\qquad\qquad\qquad
    \;-\;
    \bigl((\bar m_0-\mu)^2 + \beta_0^{ij}\bigr)\,e^{-2\kappa u}\,\Bigr]\,du.
\end{align}
Computing the three standard integrals,
\begin{align}
    \int_0^t e^{-(2\kappa+\sigma^2)(t-u)}\,du
    &\;=\; \frac{1-e^{-(2\kappa+\sigma^2)t}}{2\kappa+\sigma^2},
    \\[2pt]
    \int_0^t e^{-(2\kappa+\sigma^2)(t-u)}\,e^{-\kappa u}\,du
    &\;=\; \frac{e^{-\kappa t}-e^{-(2\kappa+\sigma^2)t}}{\kappa+\sigma^2},
    \\[2pt]
    \int_0^t e^{-(2\kappa+\sigma^2)(t-u)}\,e^{-2\kappa u}\,du
    &\;=\; \frac{e^{-2\kappa t}-e^{-(2\kappa+\sigma^2)t}}{\sigma^2},
\end{align}
and substituting them into \eqref{eq::V-duhamel} yields
\begin{align}
    \bar v_t
    &\;=\;
    \sigma^2 N^2 \mu(1-\mu)\,\frac{1-e^{-(2\kappa+\sigma^2)t}}{2\kappa+\sigma^2}
    \;+\;
    \sigma^2 N^2 (\bar m_0-\mu)(1-2\mu)\,\frac{e^{-\kappa t}-e^{-(2\kappa+\sigma^2)t}}{\kappa+\sigma^2}
    \nonumber\\
    &\quad
    \;-\;
    \sigma^2 N^2 \bigl((\bar m_0-\mu)^2 + \beta_0^{ij}\bigr)\,\frac{e^{-2\kappa t}-e^{-(2\kappa+\sigma^2)t}}{\sigma^2}.
\end{align}
Noting that the $\sigma^2$ in the third family cancels against the denominator $\sigma^2$, and factoring $N^2$,
\begin{align}
    \bar v_t
    &\;=\;
    N^2 \biggl[\,
    \frac{\sigma^2 \mu(1-\mu)}{2\kappa+\sigma^2}\bigl(1-e^{-(2\kappa+\sigma^2)t}\bigr)
    \nonumber\\
    &\qquad\;\;\;
    \;+\;
    \frac{\sigma^2 (\bar m_0-\mu)(1-2\mu)}{\kappa+\sigma^2}\bigl(e^{-\kappa t}-e^{-(2\kappa+\sigma^2)t}\bigr)
    \nonumber\\
    &\qquad\;\;\;
    \;-\;
    \bigl((\bar m_0-\mu)^2+\beta_0^{ij}\bigr)\bigl(e^{-2\kappa t}-e^{-(2\kappa+\sigma^2)t}\bigr)
    \,\biggr],
\end{align}
which is \eqref{eq::lem1-var}.
\end{proof}

\subsection{First and second moment of the renormalized process}
\label{app::lemma-renormalized}

\begin{lemma}[Moments of the renormalized SDE]
\label{lem::moments-renormalized}
Let $\tilde W_t^{ij}$ solve \eqref{eq::cell-sde-renorm} with renormalized parameters \eqref{eq::renorm-params} and deterministic initial condition $\tilde W_0^{ij} = \bar m_0$. Assume the stability condition $\tilde\kappa>0$, equivalently $|C|>\sigma^2/(2\kappa+\sigma^2)$. Then for every $t\ge 0$,
\begin{align}
\tilde m_t
&\;=\;
\mu + (\bar m_0-\mu)\,e^{-\kappa t}, \label{eq::lem2-mean}
\\[4pt]
\tilde v_t
&\;=\;
N^2 \biggl[\,
\frac{\sigma^2 \mu(1-\mu)}{2\kappa+\sigma^2}\bigl(1-e^{-(2\kappa+\sigma^2)t}\bigr)
\nonumber\\
&\qquad\;\;\;
\;+\;
\frac{\sigma^2 (\bar m_0-\mu)(1-2\mu)}{\kappa+\sigma^2}\bigl(e^{-\kappa t}-e^{-(2\kappa+\sigma^2)t}\bigr)
\nonumber\\
&\qquad\;\;\;
\;-\;
(\bar m_0-\mu)^2\bigl(e^{-2\kappa t}-e^{-(2\kappa+\sigma^2)t}\bigr)
\,\biggr]. \label{eq::lem2-var}
\end{align}
\end{lemma}

\begin{proof}
The renormalized SDE \eqref{eq::cell-sde-renorm} is a one-dimensional It\^o diffusion with state-dependent volatility $\tilde\sigma\sqrt{\tilde W_t^{ij}(1-\tilde W_t^{ij})}$ and drift $\tilde\kappa(\tilde\mu_t - \tilde W_t^{ij})$. Its first moment is fixed by the choice of the time-varying target $\tilde\mu_t$, and the variance solves a linear ODE whose coefficients align with those of the spatial-average dynamics \eqref{eq::V-ode-clean} after substituting the renormalization \eqref{eq::renorm-params}.

Taking expectations in \eqref{eq::cell-sde-renorm}, the martingale term vanishes and $\tilde m_t$ solves the linear ODE
\begin{equation}\label{eq::lem2-mean-ode}
    \frac{d}{dt}\tilde m_t
    \;=\;
    \tilde\kappa\bigl(\tilde\mu_t - \tilde m_t\bigr),
    \qquad
    \tilde m_0 \;=\; \bar m_0.
\end{equation}
Substituting the time-varying target $\tilde\mu_t = \bar m_t + (\kappa/\tilde\kappa)(\mu - \bar m_t)$ from \eqref{eq::renorm-params}, the right-hand side of \eqref{eq::lem2-mean-ode} becomes
\begin{equation}\label{eq::lem2-mean-rhs}
    \tilde\kappa\bigl(\tilde\mu_t - \tilde m_t\bigr)
    \;=\;
    \tilde\kappa\bigl(\bar m_t - \tilde m_t\bigr)
    \;+\;
    \kappa\bigl(\mu - \bar m_t\bigr).
\end{equation}
Evaluating \eqref{eq::lem2-mean-rhs} at the candidate trajectory $\tilde m_t = \bar m_t$, the first term vanishes and only $\kappa(\mu - \bar m_t)$ remains. By Lemma~\ref{lem::moments-continuous}, this is exactly $\frac{d}{dt}\bar m_t$, so $\tilde m_t = \bar m_t$ satisfies the same ODE \eqref{eq::lem2-mean-ode} as $\bar m_t$. With matching initial conditions $\tilde m_0 = \bar m_0$, uniqueness of the linear ODE yields $\tilde m_t = \bar m_t$ for all $t\ge 0$, which is \eqref{eq::lem2-mean}. The renormalization preserves the mean trajectory exactly: the time-varying target $\tilde\mu_t$ is precisely the correction needed to compensate for replacing $\kappa$ by the lowered rate $\tilde\kappa$, so that the mean tracks $\bar m_t$ rather than relaxing toward $\mu$ at the wrong rate.

We now derive the variance. Apply It\^o's formula to $(\tilde W_t^{ij})^2$,
\begin{equation}\label{eq::lem2-ito}
    d(\tilde W_t^{ij})^2
    \;=\;
    2\,\tilde W_t^{ij}\,d\tilde W_t^{ij}
    \;+\;
    (d\tilde W_t^{ij})^2.
\end{equation}
Let $\tilde q_t := \mathbb{E}[(\tilde W_t^{ij})^2]$. Multiplying \eqref{eq::cell-sde-renorm} by $2\tilde W_t^{ij}$ and taking expectations, the stochastic part has zero mean and the drift gives
\begin{equation}\label{eq::lem2-q-drift}
    \mathbb{E}\bigl[\,2\,\tilde W_t^{ij}\,d\tilde W_t^{ij}\bigr]
    \;=\;
    2\tilde\kappa\bigl(\tilde\mu_t\,\tilde m_t - \tilde q_t\bigr)\,dt.
\end{equation}
The quadratic variation of \eqref{eq::cell-sde-renorm} is $\tilde\sigma^2\,\tilde W_t^{ij}(1-\tilde W_t^{ij})\,dt$, and taking expectations,
\begin{equation}\label{eq::lem2-q-qv}
    \mathbb{E}\bigl[(d\tilde W_t^{ij})^2\bigr]
    \;=\;
    \tilde\sigma^2\bigl(\tilde m_t - \tilde q_t\bigr)\,dt.
\end{equation}
Summing \eqref{eq::lem2-q-drift} and \eqref{eq::lem2-q-qv},
\begin{equation}\label{eq::lem2-q-ode}
    \frac{d}{dt}\tilde q_t
    \;=\;
    2\tilde\kappa\bigl(\tilde\mu_t\,\tilde m_t - \tilde q_t\bigr)
    \;+\;
    \tilde\sigma^2\bigl(\tilde m_t - \tilde q_t\bigr).
\end{equation}

Subtract $\frac{d}{dt}\tilde m_t^2 = 2\tilde m_t\,\frac{d}{dt}\tilde m_t = 2\tilde\kappa\bigl(\tilde\mu_t\,\tilde m_t - \tilde m_t^2\bigr)$ from \eqref{eq::lem2-q-ode}. The cross-term $2\tilde\kappa\,\tilde\mu_t\,\tilde m_t$ cancels and, writing $\tilde m_t - \tilde q_t = \tilde m_t(1-\tilde m_t) - \tilde v_t$ via $\tilde q_t = \tilde v_t + \tilde m_t^2$, we obtain
\begin{equation}\label{eq::lem2-v-ode-canonical}
    \frac{d}{dt}\tilde v_t
    \;=\;
    -(2\tilde\kappa+\tilde\sigma^2)\,\tilde v_t
    \;+\;
    \tilde\sigma^2\,\tilde m_t(1-\tilde m_t),
    \qquad
    \tilde v_0 \;=\; 0.
\end{equation}
Equation \eqref{eq::lem2-v-ode-canonical} is the canonical Jacobi-diffusion variance ODE in the renormalized parameters, an instance of \eqref{eq::app-var-ode} of Appendix~\ref{app::jacobi-variance}. Its coefficients are not yet in the form of those of \eqref{eq::V-ode-clean}, and three distinct simplifications brought by the renormalization \eqref{eq::renorm-params} are needed to bring them into agreement.

The first simplification concerns the noise amplitude. By \eqref{eq::renorm-params}, $\tilde\sigma^2 = \sigma^2/|C| = \sigma^2 N^2$. The right-hand side of \eqref{eq::lem2-v-ode-canonical} therefore acquires the same per-cell prefactor $\sigma^2 N^2$ that the spatial-average ODE \eqref{eq::V-ode-clean} carries.

The second simplification concerns the damping coefficient. By \eqref{eq::renorm-params}, $\tilde\kappa = \kappa - \sigma^2(N^2-1)/2$, so
\begin{equation}\label{eq::damping-collapse}
    2\tilde\kappa + \tilde\sigma^2
    \;=\;
    \bigl[\,2\kappa - \sigma^2(N^2-1)\,\bigr]
    \;+\;
    \sigma^2 N^2
    \;=\;
    2\kappa + \sigma^2.
\end{equation}
The lowered mean reversion $\tilde\kappa$ is exactly the choice that makes the damping rate $2\tilde\kappa+\tilde\sigma^2$ collapse to $2\kappa+\sigma^2$, the damping rate of the spatial-average variance.

The third simplification concerns the mean. By the first moment alignment shown above, $\tilde m_t = \bar m_t$, so the nonlinear term $\tilde\sigma^2\,\tilde m_t(1-\tilde m_t)$ becomes $\sigma^2 N^2\,\bar m_t(1-\bar m_t)$, the same nonlinear term that drives \eqref{eq::V-ode-clean}.

Combining the three simplifications, the variance ODE reduces to
\begin{equation}\label{eq::lem2-v-ode-clean}
    \frac{d}{dt}\tilde v_t
    \;=\;
    -(2\kappa + \sigma^2)\,\tilde v_t
    \;+\;
    \sigma^2 N^2\,\bar m_t(1-\bar m_t),
    \qquad
    \tilde v_0 \;=\; 0.
\end{equation}
Equation \eqref{eq::lem2-v-ode-clean} is identical to the spatial-average variance ODE \eqref{eq::V-ode-clean} except for the absence of the term $-\sigma^2 N^2\,\beta_0^{ij}\,e^{-2\kappa t}$. That term explicits the in-cell variance of the initial graphon in the spatial-average dynamics. The renormalized process has no such term because its initial condition is the deterministic constant $\bar m_0$ on each cell, with no spatial variation to dissipate.

Substituting $\bar m_u = \mu + (\bar m_0-\mu)e^{-\kappa u}$ into the right-hand side of \eqref{eq::lem2-v-ode-clean} and expanding,
\begin{equation}\label{eq::lem2-m-times-1minusm}
    \bar m_u(1-\bar m_u)
    \;=\;
    \mu(1-\mu)
    \;+\;
    (\bar m_0-\mu)(1-2\mu)\,e^{-\kappa u}
    \;-\;
    (\bar m_0-\mu)^2\,e^{-2\kappa u}.
\end{equation}
The right-hand side is a sum of three exponential families with rates $\{0,\,\kappa,\,2\kappa\}$. Solving \eqref{eq::lem2-v-ode-clean} by Duhamel with the integrating factor $e^{(2\kappa+\sigma^2)t}$,
\begin{equation}\label{eq::lem2-v-duhamel}
    \tilde v_t
    \;=\;
    \int_0^t e^{-(2\kappa+\sigma^2)(t-u)}\,
    \sigma^2 N^2\,\bar m_u(1-\bar m_u)\,du.
\end{equation}
Reusing the three standard integrals computed in the proof of Lemma~\ref{lem::moments-continuous}, substituting into \eqref{eq::lem2-v-duhamel}, factoring $N^2$ as a common prefactor, and noting that the $\sigma^2$ in the third family cancels against the denominator $\sigma^2$,
\begin{align}
    \tilde v_t
    &\;=\;
    N^2 \biggl[\,
    \frac{\sigma^2 \mu(1-\mu)}{2\kappa+\sigma^2}\bigl(1-e^{-(2\kappa+\sigma^2)t}\bigr)
    \nonumber\\
    &\qquad\;\;\;
    \;+\;
    \frac{\sigma^2 (\bar m_0-\mu)(1-2\mu)}{\kappa+\sigma^2}\bigl(e^{-\kappa t}-e^{-(2\kappa+\sigma^2)t}\bigr)
    \nonumber\\
    &\qquad\;\;\;
    \;-\;
    (\bar m_0-\mu)^2\bigl(e^{-2\kappa t}-e^{-(2\kappa+\sigma^2)t}\bigr)
    \,\biggr],
\end{align}
which is \eqref{eq::lem2-var}.
\end{proof}

\subsection{Proof of Theorem~\ref{thm::moment-matching}}
\label{app::theorem-proof}

\begin{proof}[Proof of Theorem~\ref{thm::moment-matching}]
By Lemma~\ref{lem::moments-continuous}, $\bar m_t = \mu+(\bar m_0-\mu)e^{-\kappa t}$. By Lemma~\ref{lem::moments-renormalized}, $\tilde m_t = \mu+(\bar m_0-\mu)e^{-\kappa t}$. The two coincide for every $t\ge 0$, which gives the first claim of the theorem.

Subtract \eqref{eq::lem1-var} from \eqref{eq::lem2-var}. After factoring out the common $N^2$ prefactor, the exponential families at rates $1$ and $e^{-\kappa t}$ cancel exactly because their coefficients agree in the two lemmas. What remains is the family $e^{-2\kappa t}-e^{-(2\kappa+\sigma^2)t}$, with coefficient $-(\bar m_0-\mu)^2$ inside the bracket of $\tilde v_t$ and $-((\bar m_0-\mu)^2+\beta_0^{ij})$ inside the bracket of $\bar v_t$. The squared-deviation contributions $-(\bar m_0-\mu)^2$ cancel as well and the discrepancy reduces to
\begin{equation}\label{eq::thm-discrepancy-final}
    \tilde v_t - \bar v_t
    \;=\;
    \beta_0^{ij}\,N^2\,\bigl(e^{-2\kappa t}-e^{-(2\kappa+\sigma^2)t}\bigr),
\end{equation}
which is the second claim of the theorem.
\end{proof}

\subsection{Proof of Corollary~\ref{cor::forward-graph-ic}}
\label{app::corollary-proof}

\begin{proof}[Proof of Corollary~\ref{cor::forward-graph-ic}]
The substitution \eqref{eq::ic-substitution} sets the initial condition to $\tilde W_0^{ij}=A_0^{ij}$ and propagates the plug-in $\bar W_0^{ij}\to A_0^{ij}$ into the time-varying target $\tilde\mu_t$. The continuous reference $\bar W_t^{ij}$ is the cell average of the underlying graphon, with general in-cell variance $\beta_0^{ij}\ge 0$, so $\bar m_t$ and $\bar v_t$ are given by Lemma~\ref{lem::moments-continuous} with $\bar m_0=\bar W_0^{ij}$.

Conditional on $A_0^{ij}=a$, both the initial condition and the plug-in target $\tilde\mu_t$ are deterministic functions of $a$, and the SDE \eqref{eq::cell-sde-renorm} satisfies the hypotheses of Lemma~\ref{lem::moments-renormalized} with $\bar m_0$ replaced by $a$. Define the conditional moments as functions of the initial condition value,
\begin{equation}\label{eq::corproof-cond-defs}
    m_t(a) \;:=\; \mathbb{E}\bigl[\tilde W_t^{ij}\,|\,A_0^{ij}=a\bigr],
    \qquad
    v_t(a) \;:=\; \mathrm{Var}\bigl(\tilde W_t^{ij}\,|\,A_0^{ij}=a\bigr).
\end{equation}
The lemma then gives $m_t(a) = \mu+(a-\mu)e^{-\kappa t}$, and $v_t(a)$ equal to the right-hand side of \eqref{eq::lem2-var} with $\bar m_0=a$. Note that $v_t(a)$ carries no $\beta_0^{ij}$ contribution because the renormalized SDE has scalar deterministic initial condition $a$, with no spatial variation inside the cell.

By the tower property and $\mathbb{E}[A_0^{ij}]=\bar W_0^{ij}$,
\begin{equation}\label{eq::corproof-marginal-mean}
    \mathbb{E}\bigl[\tilde W_t^{ij}\bigr]
    \;=\;
    \mathbb{E}\bigl[m_t(A_0^{ij})\bigr]
    \;=\;
    \mu+(\bar W_0^{ij}-\mu)e^{-\kappa t}
    \;=\;
    \mathbb{E}\bigl[\bar W_t^{ij}\bigr],
\end{equation}
which is the first claim of the corollary.

For the variance, the law of total variance gives
\begin{equation}\label{eq::corproof-LTV}
    \mathrm{Var}\bigl(\tilde W_t^{ij}\bigr)
    \;=\;
    \mathbb{E}\bigl[v_t(A_0^{ij})\bigr]
    \;+\;
    \mathrm{Var}\bigl(m_t(A_0^{ij})\bigr).
\end{equation}
Since $m_t(A_0^{ij}) = \mu+(A_0^{ij}-\mu)e^{-\kappa t}$ and, using $A_0^{ij}\in\{0,1\}$ and $\mathbb{E}[A_0^{ij}]=\bar W_0^{ij}$, $\mathrm{Var}(A_0^{ij})=\beta_1^{ij}$, the second term equals $\beta_1^{ij}\,e^{-2\kappa t}$. For the first term, averaging $v_t(a)$ over $A_0^{ij}$ uses
\begin{equation}\label{eq::corproof-bernoulli-moments}
    \mathbb{E}\bigl[A_0^{ij}-\mu\bigr] \;=\; \bar W_0^{ij}-\mu,
    \qquad
    \mathbb{E}\bigl[(A_0^{ij}-\mu)^2\bigr] \;=\; (\bar W_0^{ij}-\mu)^2 + \beta_1^{ij}.
\end{equation}
Let $\tilde v_t^{\mathrm{det}}$ denote the variance of the renormalized SDE under the deterministic initial condition $\bar W_0^{ij}$, that is, the right-hand side of \eqref{eq::lem2-var} with $\bar m_0 = \bar W_0^{ij}$. This is exactly the renormalized variance compared to $\bar v_t$ in Theorem~\ref{thm::moment-matching}. Inside \eqref{eq::lem2-var}, the terms with rates $1$ and $e^{-\kappa t}$ are constant and linear in $a$, so they reduce after averaging to the corresponding terms of $\tilde v_t^{\mathrm{det}}$. The term with rate $e^{-2\kappa t}-e^{-(2\kappa+\sigma^2)t}$ carries $-(a-\mu)^2$ and gains an extra $-\beta_1^{ij}$ in expectation by \eqref{eq::corproof-bernoulli-moments}. Therefore
\begin{equation}\label{eq::corproof-cond-var}
    \mathbb{E}\bigl[v_t(A_0^{ij})\bigr]
    \;=\;
    \tilde v_t^{\mathrm{det}}
    \;-\;
    N^2\,\beta_1^{ij}\bigl(e^{-2\kappa t}-e^{-(2\kappa+\sigma^2)t}\bigr).
\end{equation}
Substituting \eqref{eq::corproof-cond-var} and the second-term identity into \eqref{eq::corproof-LTV},
\begin{equation}\label{eq::corproof-discrepancy-pre}
    \mathrm{Var}\bigl(\tilde W_t^{ij}\bigr) - \bar v_t
    \;=\;
    \bigl(\tilde v_t^{\mathrm{det}} - \bar v_t\bigr)
    \;+\;
    \beta_1^{ij}\bigl[\,e^{-2\kappa t} - N^2\bigl(e^{-2\kappa t}-e^{-(2\kappa+\sigma^2)t}\bigr)\,\bigr].
\end{equation}
The first parenthesis is exactly the discrepancy controlled by Theorem~\ref{thm::moment-matching}, namely $\tilde v_t^{\mathrm{det}} - \bar v_t = \beta_0^{ij}N^2\bigl(e^{-2\kappa t}-e^{-(2\kappa+\sigma^2)t}\bigr)$. Using $e^{-2\kappa t}-e^{-(2\kappa+\sigma^2)t} = e^{-2\kappa t}(1-e^{-\sigma^2 t})$ to factor $e^{-2\kappa t}$ from both contributions,
\begin{equation}\label{eq::corproof-discrepancy-final}
    \mathrm{Var}\bigl(\tilde W_t^{ij}\bigr) - \mathrm{Var}\bigl(\bar W_t^{ij}\bigr)
    \;=\;
    \beta_0^{ij}\,N^2\,e^{-2\kappa t}\bigl(1-e^{-\sigma^2 t}\bigr)
    \;+\;
    \beta_1^{ij}\,e^{-2\kappa t}\,\bigl[\,1-N^2\bigl(1-e^{-\sigma^2 t}\bigr)\,\bigr],
\end{equation}
which is the second claim of the corollary.
\end{proof}

%% file: sections/appendix_proof_density_homomorphism.tex
\section{Proof of the Erd\H{o}s--R\'enyi limit theorem}
\label{app::er-limit}

This appendix proves Theorem~\ref{thm::er-limit} of Section~\ref{sec::graph-as-ic}. We open by spelling out the role of the $t\to\infty$ limit, which is what produces edges with the structure that the rest of the argument exploits. As $t\to\infty$, each cell state $\tilde W_t^{ij}$ converges in distribution to $\mathrm{Beta}(\mu)$ and, because the SDEs \eqref{eq::cell-sde-renorm} are decoupled and driven by independent Brownian motions $\{B_t^{ij}\}$, the cells become mutually independent. Both facts are classical for the Jacobi diffusion (Appendix~\ref{app::jacobi}, in particular the asymptotic moments of \ref{app::jacobi-asymptotic} and the stationary law of \ref{app::jacobi-stationary}). Through the Bernoulli sampling \eqref{eq::structured-sampling}, this turns the edge family $\{A_t^{ij}\}_{i,j=1}^N$ into mutually independent Bernoulli variables with marginal expectation
\begin{equation}\label{eq::edge-mean}
    \lim_{t\to\infty}\mathbb E\bigl[A_t^{ij}\bigr]
    \;=\;
    \lim_{t\to\infty}\mathbb E\bigl[\mathbb E[A_t^{ij}\mid \tilde W_t^{ij}]\bigr]
    \;=\;
    \lim_{t\to\infty}\mathbb E\bigl[\tilde W_t^{ij}\bigr]
    \;=\;
    \mu.
\end{equation}
The $N\to\infty$ limit is what then concentrates the homomorphism density of the resulting graph around its mean.
We write $A_t=(A_t^{ij})_{i,j=1}^N$ for the graph produced by the pipeline \eqref{eq::structured-sampling}, and denote by $\mathrm{hom}(H,A_t)$ the number of edge-preserving maps $V(H)\to\{1,\dots,N\}$, so that $t(H,A_t)=\mathrm{hom}(H,A_t)/N^{|V(H)|}$.

The concentration step relies on the following classical inequality, which we state for completeness and reference precisely from the proof.

\begin{lemma}[McDiarmid's inequality, \citealp{mcdiarmid1989method}]\label{lem::mcdiarmid}
    Let $X_1,\dots,X_m$ be independent random variables with $X_k$ taking values in a measurable space $\mathcal X_k$. Let $f:\prod_{k=1}^m\mathcal X_k\to\mathbb R$ be a function and let $c_k\ge 0$ denote the maximal change in $f$ when its $k$-th input is replaced while all the other inputs are held fixed, that is,
    \begin{equation}\label{eq::bounded-diff}
        c_k
        \;:=\;
        \sup_{x_1,\dots,x_m,\,x_k'}
        \bigl|\,f(x_1,\dots,x_k,\dots,x_m) - f(x_1,\dots,x_k',\dots,x_m)\,\bigr|,
        \qquad k=1,\dots,m.
    \end{equation}
    Then, for every $\epsilon>0$,
    \begin{equation}\label{eq::mcdiarmid-bound}
        \mathbb P\bigl(\,\bigl|\,f(X_1,\dots,X_m) - \mathbb E[f(X_1,\dots,X_m)]\,\bigr|\ge\epsilon\,\bigr)
        \;\le\;
        2\exp\!\biggl(-\frac{2\epsilon^2}{\sum_{k=1}^m c_k^2}\biggr).
    \end{equation}
\end{lemma}

\begin{proof}[Proof of Theorem~\ref{thm::er-limit}]
We restrict the analysis throughout to injective maps $\phi:V(H)\to\{1,\dots,N\}$. The non-injective contribution to the homomorphism density vanishes in the dense limit $N\to\infty$, by the standard equivalence of injective and non-injective homomorphism densities~\cite{lovasz2012large}.

\textbf{Expected homomorphism density.} Writing $\mathrm{hom}(H,A_t)$ as a sum of edge-indicator products over injective maps $\phi$,
\begin{equation}\label{eq::density-sum-pre}
    t(H,A_t)
    \;=\;
    \frac{1}{N^{|V(H)|}}
    \sum_{\phi}\,
    \prod_{\{u,v\}\in E(H)} A_t^{\phi(u)\phi(v)}.
\end{equation}
For a homomorphism $\phi$, distinct edges $\{u,v\}\neq\{u',v'\}$ of $H$ are sent to distinct unordered pairs $\{\phi(u),\phi(v)\}\neq\{\phi(u'),\phi(v')\}$ of $A_t$, so the family $\{A_t^{\phi(u)\phi(v)}\}_{\{u,v\}\in E(H)}$ entering the inner product of \eqref{eq::density-sum-pre} consists of $|E(H)|$ pairwise distinct edge indicators of $A_t$. By the stationary independence stated in \eqref{eq::edge-mean} and the surrounding paragraph, these indicators are mutually independent in the limit $t\to\infty$, and the expectation factorizes,
\begin{equation}\label{eq::factorization}
    \lim_{t\to\infty}\mathbb{E}\!\left[\,\prod_{\{u,v\}\in E(H)} A_t^{\phi(u)\phi(v)}\,\right]
    \;=\;
    \prod_{\{u,v\}\in E(H)} \lim_{t\to\infty}\mathbb{E}\bigl[A_t^{\phi(u)\phi(v)}\bigr]
    \;=\;
    \mu^{|E(H)|},
\end{equation}
independently of the particular $\phi$. Taking expectations on both sides of \eqref{eq::density-sum-pre} and summing \eqref{eq::factorization} gives
\begin{equation}\label{eq::expectation-final}
    \lim_{t\to\infty}\mathbb E\bigl[t(H,A_t)\bigr]
    \;=\;
    \mu^{|E(H)|}.
\end{equation}

\textbf{Concentration.} We view $t(H,A_t)$ as a function of the independent edge indicators $\{A_t^{kl}\}_{k\le l}$ at stationarity, with the entries $A_t^{lk}=A_t^{kl}$ for $k>l$ determined by symmetry. For a fixed unordered edge $\{k,l\}$ with $k\le l$, the maps $\phi$ contributing a term in \eqref{eq::density-sum-pre} that depends on $A_t^{kl}$ are exactly those for which some edge $\{u,v\}\in E(H)$ satisfies $\{\phi(u),\phi(v)\}=\{k,l\}$. The number of such maps is at most $2\,|E(H)|\,N^{|V(H)|-2}$: there are $|E(H)|$ choices for the edge of $H$, $2$ orientations of $\phi$ on its endpoints, and $N^{|V(H)|-2}$ free assignments for the remaining vertices of $H$. Each affected term in \eqref{eq::density-sum-pre} contributes a value in $\{0,1\}$ and is normalized by $1/N^{|V(H)|}$, so flipping the indicator $A_t^{kl}$ changes $t(H,A_t)$ by at most
\begin{equation}\label{eq::bd-edge}
    c_{kl}
    \;:=\;
    \frac{2\,|E(H)|\,N^{|V(H)|-2}}{N^{|V(H)|}}
    \;=\;
    \frac{2\,|E(H)|}{N^2},
\end{equation}
which is the bounded-differences modulus of $t(H,A_t)$ along the $\{k,l\}$ coordinate, in the sense of \eqref{eq::bounded-diff}. There are $N(N+1)/2\le N^2$ independent edge variables $\{A_t^{kl}\}_{k\le l}$, so
\begin{equation}\label{eq::bd-sum}
    \sum_{k\le l} c_{kl}^2
    \;\le\;
    N^2\cdot\frac{4\,|E(H)|^2}{N^4}
    \;=\;
    \frac{4\,|E(H)|^2}{N^2}.
\end{equation}
Applying Lemma~\ref{lem::mcdiarmid} to $f(\{A_t^{kl}\}_{k\le l}):=t(H,A_t)$ with bounded-differences constants \eqref{eq::bd-edge} and the squared-sum bound \eqref{eq::bd-sum} gives, for every $\epsilon>0$ and at stationarity,
\begin{equation}\label{eq::mcdiarmid-application}
    \lim_{t\to\infty}\mathbb P\bigl(\,\bigl|\,t(H,A_t) - \mathbb E[t(H,A_t)]\,\bigr|\ge\epsilon\,\bigr)
    \;\le\;
    2\exp\!\biggl(-\frac{N^2\,\epsilon^2}{2\,|E(H)|^2}\biggr).
\end{equation}
By \eqref{eq::expectation-final}, $\lim_{t\to\infty}\mathbb E[t(H,A_t)]$ goes to $\mu^{|E(H)|}$, so
\begin{equation}\label{eq::final-conc}
    \lim_{t\to\infty}\mathbb P\bigl(\,\bigl|\,t(H,A_t) - \mu^{|E(H)|}\,\bigr|\ge\epsilon\,\bigr)
    \;\le\;
    2\exp\!\biggl(-\frac{N^2\,\epsilon^2}{2\,|E(H)|^2}\biggr).
\end{equation}
The right-hand side of \eqref{eq::final-conc} tends to zero as $N\to\infty$, which is the convergence in probability of $t(H,A_t)$ to $\mu^{|E(H)|}$ stated in \eqref{eq::er-mean}.
\end{proof}

%% file: sections/appendix_algorithms.tex
\section{Training and sampling algorithms}
\label{app::algorithms}

This appendix collects the practical procedures that turn the constructions of Section~\ref{sec::graphon_to_graph} and Section~\ref{sec::reverse} into runnable algorithms. Algorithm~\ref{alg::training} describes the training loop that minimizes the binary denoising objective \eqref{eq::dsm-loss} over a dataset of graphs, and Algorithm~\ref{alg::sampling} describes the sampling procedure that integrates the reverse-time SDE \eqref{eq::reverse-sde} from the stationary law back to a graph at $t=0$. Both algorithms operate on the Jacobi diffusion \eqref{eq::cell-sde-renorm} with parameters \eqref{eq::renorm-params} that can be generally specified, and both reduce the only intractable component of the dynamics, namely the marginal score, to evaluations of the conditional score \eqref{eq::cond-score} at an estimator of the clean graph produced by the denoiser $\Phi_\theta$.

\subsection{Training}
\label{app::training}

Training instantiates the denoising-by-prediction view of Section~\ref{sec::learning}. Given an observed graph $A_0\in\{0,1\}^{N\times N}$ from the dataset and a diffusion time $t$ sampled uniformly on $[0,T]$, the SDE \eqref{eq::cell-sde-renorm} is integrated forward from the graph initial condition \eqref{eq::ic-substitution} via Euler--Maruyama. The resulting noisy state $\tilde W_t$ is fed to the denoiser $\Phi_\theta$, and the per-edge probabilities $\Phi_\theta(\tilde W_t,t)^{ij}$ are compared to the clean entries $A_0^{ij}$ through the binary cross-entropy of \eqref{eq::dsm-loss}. The procedure is summarized in Algorithm~\ref{alg::training}.

\begin{algorithm}[H]
\caption{Training the denoiser $\Phi_\theta$}
\label{alg::training}
\begin{algorithmic}[1]
\Require dataset $\mathcal{D}$ of graphs on $N$ nodes. SDE parameters $(\tilde\kappa,\tilde\sigma)$ and time-varying target $\tilde\mu_t(\cdot)$. Horizon $T$. Learning rate $\eta$. Batch size $B$.
\Repeat
    \State Sample mini-batch $\{A_0^{(b)}\}_{b=1}^{B}$ from $\mathcal{D}$
    \For{$b=1,\dots,B$}
        \State Sample $t\sim \mathrm{Uniform}[0,T]$
        \State Initialize $\tilde W_0^{ij}\gets A_0^{(b),ij}$ for all $i,j$
        \State Simulate the forward SDE from $\tilde W_0$ up to time $t$ to obtain $\tilde W_t^{(b)}$
        \State Store the triple $\bigl(t, \tilde W_t^{(b)}, A_0^{(b)}\bigr)$
    \EndFor
    \State Evaluate $\hat A_{t,\theta}^{(b)}\gets \Phi_\theta\bigl(\tilde W_t^{(b)}, t\bigr)$ for $b=1,\dots,B$
    \State Compute the empirical denoising loss
    \Statex \hspace{1.0em}$\widehat{\mathcal{L}}(\theta) \gets -\,\dfrac{1}{B}\sum_{b=1}^{B}\sum_{i,j=1}^{N}\Bigl[A_0^{(b),ij}\log \hat A_{t,\theta}^{(b),ij} + (1-A_0^{(b),ij})\log\bigl(1-\hat A_{t,\theta}^{(b),ij}\bigr)\Bigr]$
    \State Update $\theta\gets \theta - \eta\,\nabla_\theta\,\widehat{\mathcal{L}}(\theta)$
\Until{converged}
\State \Return $\theta$
\end{algorithmic}
\end{algorithm}

\subsection{Sampling}
\label{app::sampling}

Sampling instantiates the reverse-time integration of Section~\ref{sec::reverse-sde}.
The state at $t=T$ is initialized from the stationary law $\mathrm{Beta}(\mu)$, with density recorded in \eqref{eq::app-beta-density} of Appendix~\ref{app::jacobi-stationary}, with cells drawn independently.
The reverse-time SDE \eqref{eq::reverse-sde} is then integrated backwards in time via Euler--Maruyama, with the marginal score replaced at every step by the conditional score \eqref{eq::cond-score} evaluated at the denoiser's estimator $\hat A_0 = \Phi_\theta(\tilde W_t, t)$ of the clean graph, as discussed in Section~\ref{sec::reverse} and Section~\ref{sec::learning}.
After reaching $t=0$, the final state $\tilde W_0$ is converted into a binary graph through the per-edge Bernoulli draw of \eqref{eq::structured-sampling}.
The procedure is summarized in Algorithm~\ref{alg::sampling}.

\begin{algorithm}[H]
\caption{Sampling a graph from the trained model}
\label{alg::sampling}
\begin{algorithmic}[1]
\Require trained denoiser $\Phi_\theta$. SDE parameters $(\tilde\kappa,\tilde\sigma)$ and time-varying target $\tilde\mu_t(\cdot)$. Stationary mean $\mu$. Horizon $T$. Number of reverse steps $K$. Number of nodes $N$. Jacobi-mode truncation order $M$.
\State Set $\Delta t\gets T/K$ and $t_K\gets T$
\State Sample $\tilde W^{ij}_{T}\sim \mathrm{Beta}(\mu)$ independently for all $i,j$
\For{$k=K,K-1,\dots,1$}
    \State Set $t\gets k\,\Delta t$
    \State Evaluate $\hat A_0^{ij}\gets \Phi_\theta\bigl(\tilde W_t, t\bigr)^{ij}$ for all $i,j$
    \State Compute the conditional score $s^{ij}\gets \nabla_W\log p_t\bigl(\tilde W_t^{ij}\,\big|\,\hat A_0^{ij}\bigr)$, truncated at $n\le M$
    \State Sample $\Delta\bar B^{ij}\sim \mathcal{N}(0,\Delta t)$ independently for all $i,j$
    \State Update
    \Statex \hspace{1.0em}$\tilde W^{ij}_{t-\Delta t}\gets \tilde W^{ij}_{t} - \Bigl[\,\tilde\kappa\bigl(\tilde\mu_t(\hat A_0^{ij}) - \tilde W^{ij}_{t}\bigr) - \tilde\sigma^2\,\tilde W^{ij}_{t}(1-\tilde W^{ij}_{t})\,s^{ij}\,\Bigr]\Delta t$
    \Statex \hspace{6.5em}$+\;\tilde\sigma\,\sqrt{\tilde W^{ij}_{t}(1-\tilde W^{ij}_{t})}\,\Delta\bar B^{ij}$
    \State Clip $\tilde W^{ij}_{t-\Delta t}$ to $[0,1]$
\EndFor
\State Sample $A_0^{ij}\sim \mathrm{Ber}\bigl(\tilde W_0^{ij}\bigr)$ independently for all $i,j$
\State \Return $A_0$
\end{algorithmic}
\end{algorithm}

On the practical side, the conditional score \eqref{eq::cond-score} is a closed-form function of the noisy state, the diffusion time, and the plug-in graphon entry $\hat A_0^{ij}$.
The series in numerator and denominator of \eqref{eq::cond-score} can be truncated to a finite number of Jacobi modes, since each mode carries the exponential factor $e^{-\lambda_n t}$ with $\lambda_n$ growing quadratically in $n$ as recorded in \eqref{eq::app-jacobi-eigenvalue}, so the truncation error decays rapidly with the truncation order at every fixed $t>0$.
The same SDE parameters $(\tilde\kappa,\tilde\sigma,\tilde\mu_t)$ used at training are used at sampling.
The number of reverse-time steps $K$ controls the discretization error of the reverse simulator, and more advanced stochastic solvers such as Milstein or stochastic Heun~\cite{kloeden1992numerical, karras2022elucidating} can be used in place of Euler--Maruyama to trade per-step compute for reduced discretization error.

%% file: sections/appendix_parameterization.tex
\section{Denoiser parameterization}
\label{app::parameterization}

This appendix describes the neural-network parameterization of the denoiser $\Phi_\theta$ that appears in the training loss \eqref{eq::dsm-loss} and is queried at every reverse-time step in Algorithm~\ref{alg::sampling}. We adopt the same parameterization as DeFoG~\cite{qin2025defog}, namely the graph transformer of DiGress~\cite{vignac2023digress} augmented with Relative Random Walk Probability features~\cite{ma2023grit, siraudin2024cometh}. The role of the network and its training signal are dictated by Section~\ref{sec::learning}. The present appendix only documents the architectural choices that turn $\Phi_\theta$ into a concrete map from a noisy state and a diffusion time to a per-edge probability matrix.

\subsection{Input and output}
\label{app::param-io}

At every call, the network receives the noisy state $\tilde W_t = (\tilde W_t^{ij})_{i,j=1}^{N}$ produced by the forward simulator of \eqref{eq::cell-sde-renorm} and the diffusion time $t\in[0,T]$, and it returns a symmetric matrix $\Phi_\theta(\tilde W_t,t)\in[0,1]^{N\times N}$ whose entry $\Phi_\theta(\tilde W_t,t)^{ij}$ is the predicted Bernoulli probability of an edge between nodes $i$ and $j$. The diagonal is masked since the diffusion does not model self-loops. The same input--output convention is used at training, where the prediction is compared to the clean graph $A_0$ via the binary cross-entropy of \eqref{eq::dsm-loss}, and at sampling, where the prediction supplies the plug-in graphon estimator $\hat A_0^{ij}$ that enters the conditional score \eqref{eq::cond-score} and the time-varying target \eqref{eq::renorm-params}.

The continuous-valued state $\tilde W_t^{ij}\in[0,1]$ replaces the categorical noisy edge of the original DeFoG/DiGress formulation. In our setting the edge variable carries a single soft probability rather than a one-hot vector over discrete edge classes, so the per-edge feature dimension at the input is one and the output head has a single sigmoid logit per cell. Apart from this width adjustment the layer structure is unchanged.

\subsection{Graph transformer}
\label{app::param-transformer}

The backbone is the graph transformer of \citet{vignac2023digress}, which DeFoG inherits without modification. The network maintains three streams of representations along its $L$ layers, namely node features $\mathbf{X}\in\mathbb{R}^{N\times d_X}$, edge features $\mathbf{E}\in\mathbb{R}^{N\times N\times d_E}$, and a graph-level vector $\mathbf{y}\in\mathbb{R}^{d_y}$. Each layer applies a multi-head attention block in which queries, keys, and values are computed from node features and the attention logits are biased by a learned function of the corresponding edge feature, so that pairwise structure modulates the attention pattern explicitly. The edge features are themselves updated by a learned function of the pre-attention pair representation, and the graph-level vector is updated by pooling node and edge features through a permutation-invariant aggregator and a learned MLP, with the result broadcast back into the node and edge updates via FiLM-style affine modulation. The three streams are interleaved across layers and the whole stack is permutation equivariant by construction, since attention, pooling, and the FiLM modulations commute with simultaneous permutations of the node and edge axes.

The diffusion time $t$ is encoded by a sinusoidal embedding followed by an MLP and concatenated to the graph-level vector $\mathbf{y}$ at the input layer. The FiLM modulation then propagates the time signal into every node and edge update at every layer. The output head reads the final edge stream, applies a learned linear map followed by a sigmoid, and symmetrizes the result, yielding the per-edge probability matrix $\Phi_\theta(\tilde W_t,t)$. We refer to \citet{vignac2023digress} for the detailed equations of the attention block, the FiLM modulations, and the pooling operators.

\subsection{Relative random walk probability features}
\label{app::param-rrwp}

Graph transformers, like other graph neural networks, have limited expressive power without explicit positional information. Following DeFoG, we augment node and edge features with Relative Random Walk Probability (RRWP) encodings~\cite{ma2023grit, siraudin2024cometh}, which are a more efficient and more expressive alternative to the spectral and cycle features used in \citet{vignac2023digress}. Given the binarized noisy graph at the current step, denote its adjacency by $\hat A\in\{0,1\}^{N\times N}$, its degree matrix by $D=\mathrm{diag}(\hat A\mathbf{1})$, and the degree-normalized adjacency by $M = D^{-1}\hat A$. The RRWP tensor of order $K$ collects the first $K$ powers of $M$,
\begin{equation}\label{eq::rrwp-tensor}
    \mathrm{RRWP}(\hat A)
    \;:=\;
    \bigl[\,I,\;M,\;M^2,\;\dots,\;M^{K-1}\,\bigr]
    \;\in\;
    \mathbb{R}^{N\times N\times K},
\end{equation}
so that entry $(i,j,k)$ is the probability of reaching node $j$ from node $i$ in exactly $k$ steps of the simple random walk on $\hat A$. The diagonal slice $\bigl(\mathrm{RRWP}(\hat A)^{ii}_k\bigr)_{k=0}^{K-1}$ is a per-node feature recording return probabilities at horizons $0,\dots,K-1$ and is concatenated to the node embedding of $i$. The off-diagonal slices $\bigl(\mathrm{RRWP}(\hat A)^{ij}_k\bigr)_{k=0}^{K-1}$ for $i\neq j$ are per-edge features recording multi-hop reachabilities and are concatenated to the edge embedding of $(i,j)$. The construction is permutation equivariant in the same sense as the rest of the architecture, since powers of the normalized adjacency commute with simultaneous permutations of the node axes.

Since the state $\tilde W_t^{ij}\in[0,1]$ is continuous-valued, we obtain $\hat A$ for the RRWP construction by sampling, namely by drawing $\hat A^{ij}\sim\mathrm{Ber}(\tilde W_t^{ij})$ independently across cells via the structured Bernoulli step \eqref{eq::structured-sampling}. This is a feature of the proposed forward process rather than an architectural workaround: the dynamics produce a graph-valued sample at every time $t$ along the trajectory, so the binary input expected by \eqref{eq::rrwp-tensor} is available natively without thresholding or projection, and the RRWP features are computed on a true graph rather than on a binarized continuous state. The RRWP tensor is recomputed from scratch at every call of $\Phi_\theta$, since the noisy graph changes both during the forward simulator at training and during the reverse integration at sampling. The order $K$ is a hyperparameter that controls how many random-walk horizons enter the features and trades off expressivity against computational cost. The cost of constructing \eqref{eq::rrwp-tensor} is dominated by $K-1$ matrix products of $N\times N$ matrices and is negligible relative to the cost of a transformer forward pass.

%% file: sections/appendix_experiments.tex
\section{Experiments}
\label{app::experiments}

Most experiments were run on a workstation equipped with one AMD EPYC 9634 84-core processor,
168 CPU threads, and two NVIDIA GeForce RTX 4090 GPUs with 24 GB of memory each.
The primary workstation used NVIDIA driver 560.35.03 with CUDA 12.6.
Experiments involving the largest graph sizes were run on NVIDIA A100 GPUs with 80 GB of memory.

\subsection{Moment matching}
\label{app::exp::setup}

The setup consists on stochastic block model (SBM) with two equal-size communities,
\begin{equation}\label{eq::app-sbm}
W_0(x,y) \;=\;
\begin{cases}
p & \text{if both } x,y\in[0,\tfrac12) \text{ or both } x,y\in[\tfrac12,1], \\
q & \text{otherwise,}
\end{cases}
\end{equation}
with $p=0.7$ and $q=0.2$, so we assume stationary mean $\mu=(p+q)/2=0.45$. Even $N$ aligns the cell grid to the community boundary at $x=\tfrac12$, so the in-cell variance $\beta_0^{ij}$ vanishes in every cell.
We report two cell types: an \emph{intra-community} cell with $\bar W_0^{ij}=p=0.7$ at indices $(N/4,N/4)$, and an \emph{inter-community} cell with $\bar W_0^{ij}=q=0.2$ at indices $(N/4,N/2+N/4)$.

We set $\kappa=1$, $\sigma=0.005$ fixed across $N$, horizon $T=4$, and integrate \eqref{eq::cell-sde-renorm} by Euler--Maruyama with $n_\text{steps}=400$ ($\Delta t=0.01$). Moments are estimated from $M=10^{3}$ Monte Carlo trajectories per configuration. The choice $\sigma=0.005$ keeps the renormalized stability $\tilde\kappa>0$.

Each figure shows four panels in a single row, for one cell at a single $N$. From left to right: empirical mean of $\tilde W_t^{ij}$ under the deterministic initial condition $\tilde W_0^{ij}=\bar W_0^{ij}$; empirical variance under the same initial condition; empirical mean under the Bernoulli initial condition $\tilde W_0^{ij}=A_0^{ij}\sim\mathrm{Ber}(\bar W_0^{ij})$; empirical variance under the Bernoulli initial condition. Theoretical curves are overlaid: \emph{Ideal Cont.} is the closed form of Lemma~\ref{lem::moments-continuous}; \emph{Ideal Disc.} is the closed form of Lemma~\ref{lem::moments-renormalized}; \emph{Ideal Bern.} adds the Corollary~\ref{cor::forward-graph-ic} term to \emph{Ideal Disc.}. Since $\beta_0^{ij}=0$ for the SBM at even $N$, \emph{Ideal Cont.} and \emph{Ideal Disc.} coincide.

\textbf{Intra-community cell ($\bar W_0^{ij}=0.7$)}
\label{app::exp::intra}

\begin{figure}[H]
\centering
\includegraphics[width=\linewidth]{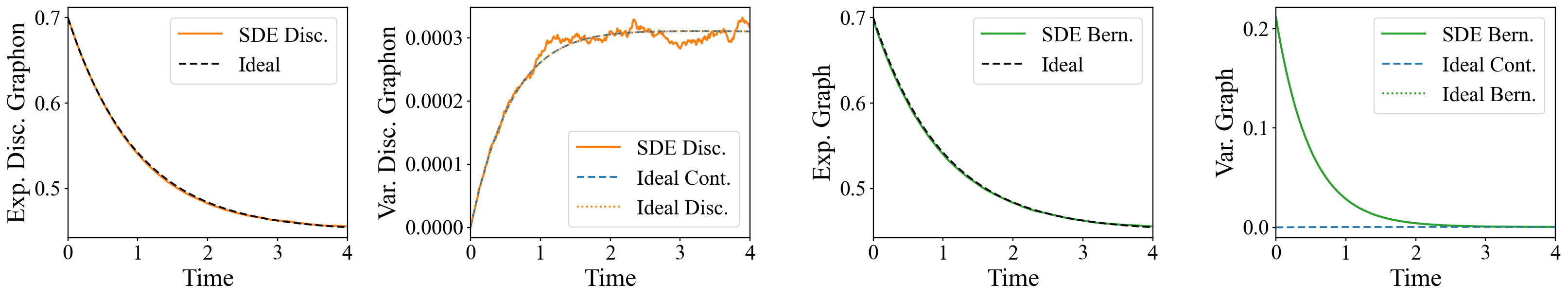}
\caption{SBM, intra-community cell, $N=10$.}
\label{fig::app-sbm-intra-N10}
\end{figure}

\begin{figure}[H]
\centering
\includegraphics[width=\linewidth]{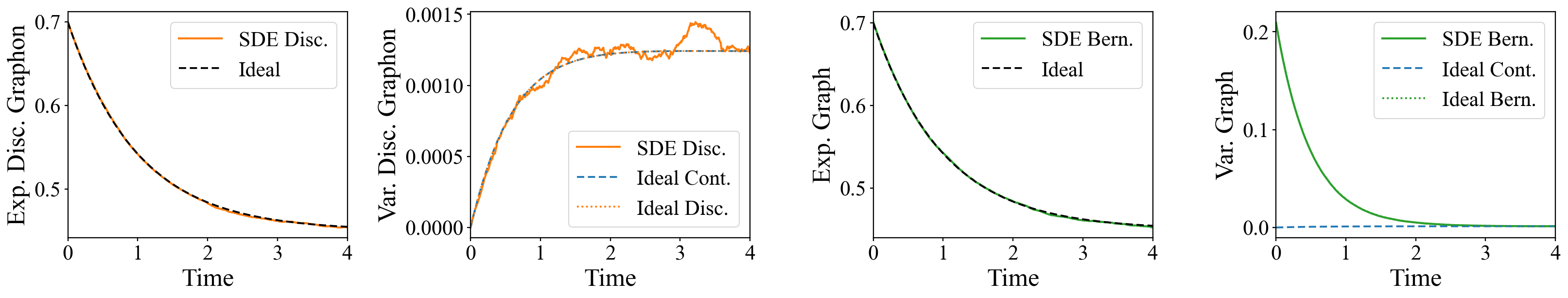}
\caption{SBM, intra-community cell, $N=20$.}
\label{fig::app-sbm-intra-N20}
\end{figure}

\begin{figure}[H]
\centering
\includegraphics[width=\linewidth]{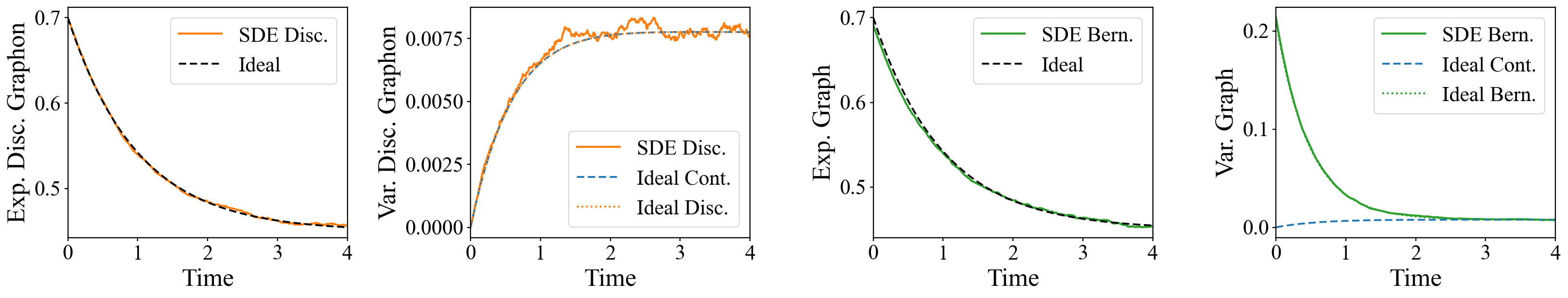}
\caption{SBM, intra-community cell, $N=50$.}
\label{fig::app-sbm-intra-N50}
\end{figure}

\begin{figure}[H]
\centering
\includegraphics[width=\linewidth]{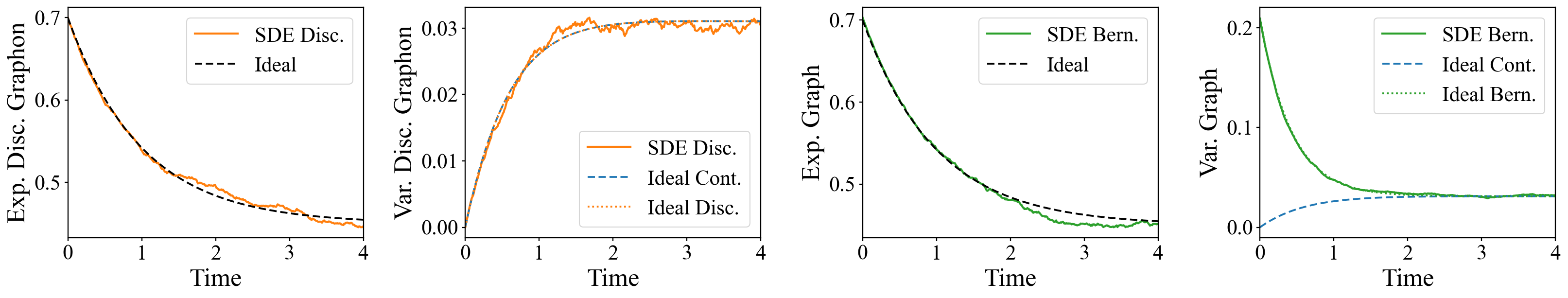}
\caption{SBM, intra-community cell, $N=100$.}
\label{fig::app-sbm-intra-N100}
\end{figure}

\textbf{Inter-community cell ($\bar W_0^{ij}=0.2$)}
\label{app::exp::inter}

\begin{figure}[H]
\centering
\includegraphics[width=\linewidth]{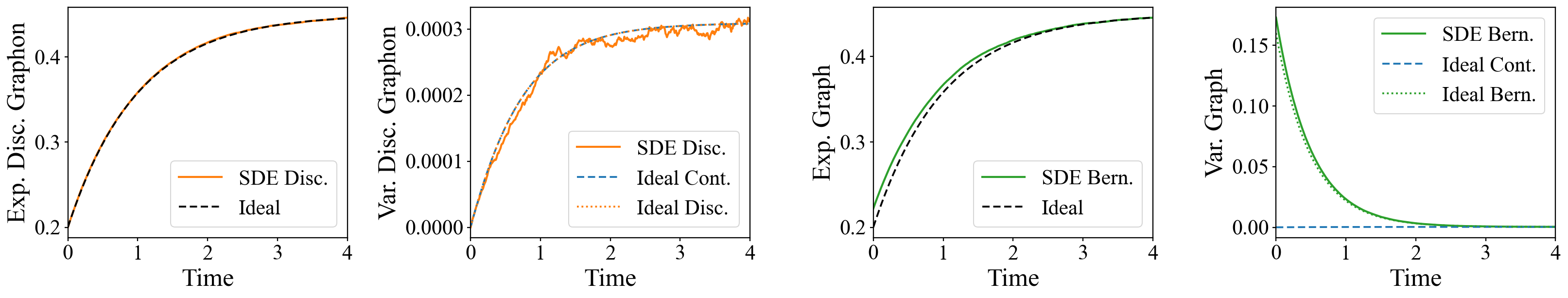}
\caption{SBM, inter-community cell, $N=10$.}
\label{fig::app-sbm-inter-N10}
\end{figure}

\begin{figure}[H]
\centering
\includegraphics[width=\linewidth]{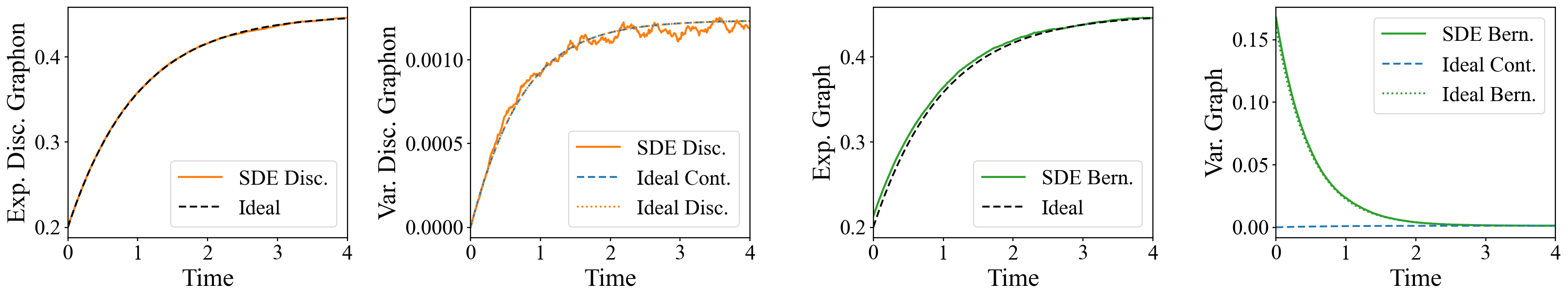}
\caption{SBM, inter-community cell, $N=20$.}
\label{fig::app-sbm-inter-N20}
\end{figure}

\begin{figure}[H]
\centering
\includegraphics[width=\linewidth]{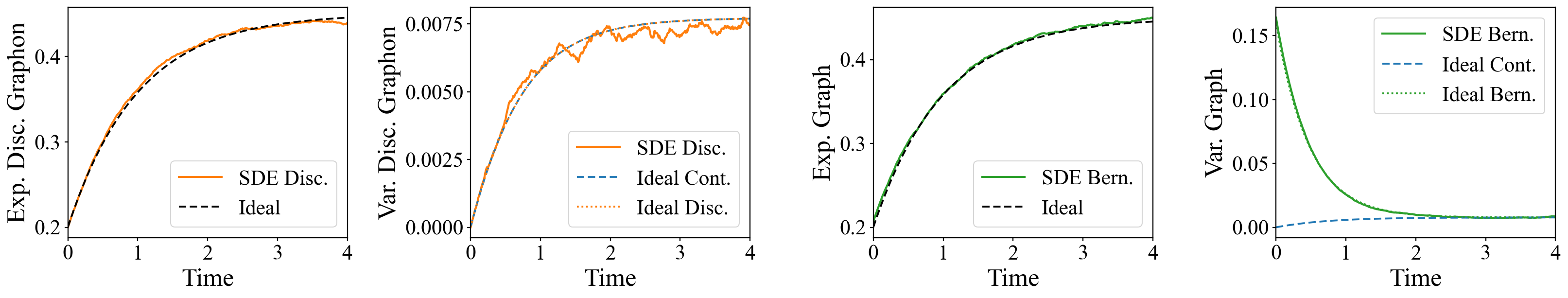}
\caption{SBM, inter-community cell, $N=50$.}
\label{fig::app-sbm-inter-N50}
\end{figure}

\begin{figure}[H]
\centering
\includegraphics[width=\linewidth]{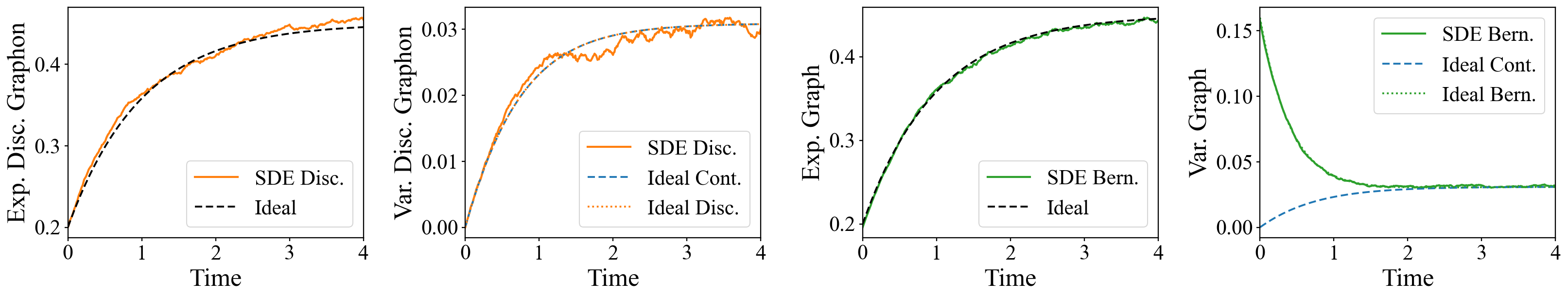}
\caption{SBM, inter-community cell, $N=100$.}
\label{fig::app-sbm-inter-N100}
\end{figure}

\subsection{Scalability experiments}

The scalability experiment evaluates whether DiPhon can generate graphs at node sizes outside the training distribution. To this end, we generate all datasets ourselves using standard graph generation tools such as \texttt{networkx}. For each graph family, we build a training set of 300 graphs, a validation set of 200 graphs, and multiple test sets of 100 graphs each. The test sets are generated across node sizes, increasing from the training regime up to the maximum size in steps of 20 nodes. We consider the following three synthetic graph families:
\begin{enumerate}
    \item \textbf{Stochastic block model (SBM).} The training graphs have node sizes ranging from 40 to 80 and are generated from a two-community SBM with within-community connection probability $0.4$ and across-community connection probability $0.005$. This family is the most naturally aligned with the graphon setting, since it admits a clear piecewise-constant graphon representation. We do not consider graphs smaller than 40 nodes because, at that scale, it becomes less robust to assess whether a graph displays the expected SBM structure.

    \item \textbf{Preferential attachment (PA).} The training graphs have node sizes ranging from 40 to 80. Although PA graphs do not admit a standard graphon representation directly, their line graphs do, which makes them a meaningful intermediate case. As with SBMs, we do not consider graphs smaller than 40 nodes because, for very small graphs, the characteristic structural signature of preferential attachment is less reliably assessed.

    \item \textbf{Tree.} The training graphs have node sizes ranging from 20 to 80. Trees do not admit an associated graphon in the classical sense, so they provide a useful test of performance clearly outside the assumptions under which the theory is derived. In this case, smaller graphs can still be included because it is structurally easy to verify whether a graph is a tree.
\end{enumerate}

The minimum graph size differs across datasets according to the corresponding accuracy metric.
For SBM and PA graphs, both training and evaluation start at $40$ nodes, since correctness is assessed through statistical tests that are less stable at very small graph sizes.
For trees, we start at $20$ nodes, because acyclicity is a deterministic property that can be reliably checked even at smaller sizes.

For DiPhon, we use a constant-parameter Jacobi diffusion. This is the simplest implementation of the proposed forward process: it acts pointwise on edge probabilities, remains bounded in $[0,1]$, and preserves the first-moment matching property of the graphon-level dynamics. We use this parametrization to isolate the effect of the proposed diffusion mechanism without adding extra modeling complexity.

All methods were trained on the corresponding training datasets and carefully tuned before evaluating out-of-distribution size generalization. For each baseline, we performed an exhaustive grid search over the diffusion hyperparameters and selected configurations that achieved strong performance within the node-size regime observed during training. This ensures that failures at larger graph sizes reflect limitations in size extrapolation rather than poor in-distribution fitting. DiPhon was tuned under the same protocol, using only graphs whose sizes belonged to the training distribution.

For each test node size, we generated $100$ graphs and report the corresponding accuracy.
Thus, every point in Figure~\ref{fig::ood_accuracy} is computed from $100$ generated samples.
A generated graph is counted as correct when it satisfies the defining structural property of the corresponding graph family. For trees, we use forest accuracy: a sample is considered correct if all its connected components are trees, so that nearly correct acyclic samples are not penalized only for missing a small number of edges.
For SBMs and PAs we use statistical tests to assess whether the models are adequate, with details that can be found in the code repository.

For DiPhon, the reverse-time dynamics were simulated using one of three candidate numerical solvers. The solver was selected using only the training-size validation regime and then kept fixed for larger test sizes. We used Heun's method for SBMs, where preserving block-level edge statistics is the main requirement. For PA graphs and trees, we used the Milstein integrator, which provided more stable extrapolation of structural constraints.

\textbf{Graphs at different resolutions.} We also provide qualitative visualizations of generated samples for the three datasets, comparing different graph sizes and baseline methods. These visualizations illustrate how the generated structures evolve as the target node size increases, and provide an additional qualitative assessment of size extrapolation. Consistently with the quantitative results, DiPhon produces the most faithful samples across datasets and node-size regimes, preserving the characteristic structure of SBMs, PA graphs, and trees more reliably than the competing methods. Some interesting aspects of the visualizations is that PA generates GruM obtains good performance metrics, but overestimates the number of edges with respect to the original data distribution. 
Also, it can be observed that DiPhon generates better trees for large node sizes, but sometimes it generates small disconnected trees.

\begin{figure}[H]
\centering
\includegraphics[width=0.9\linewidth]{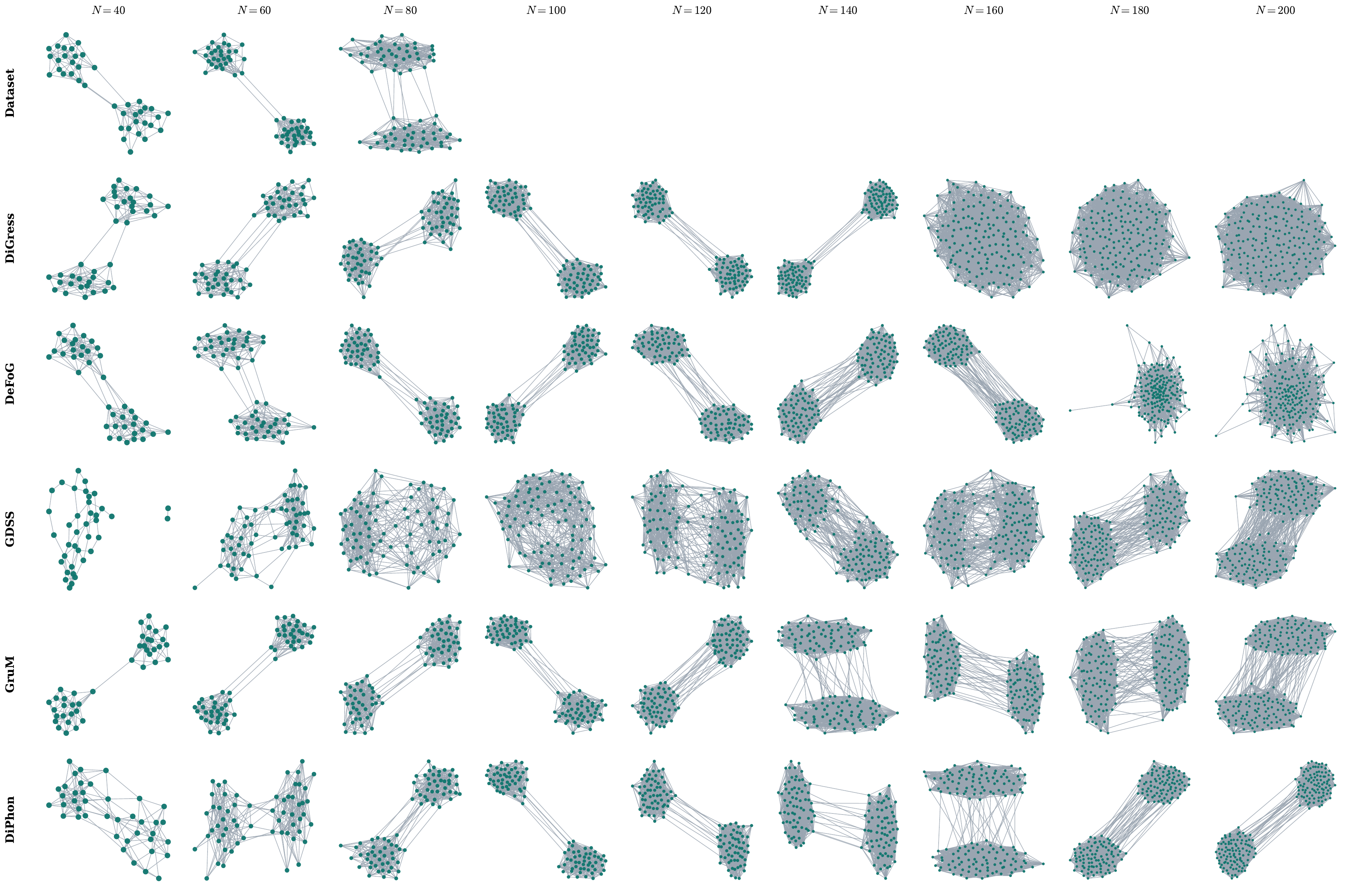}
\caption{Uncurated SBM graph samples.}
\label{fig::app-sbm-examples}
\end{figure}

\begin{figure}[H]
\centering
\includegraphics[width=0.9\linewidth]{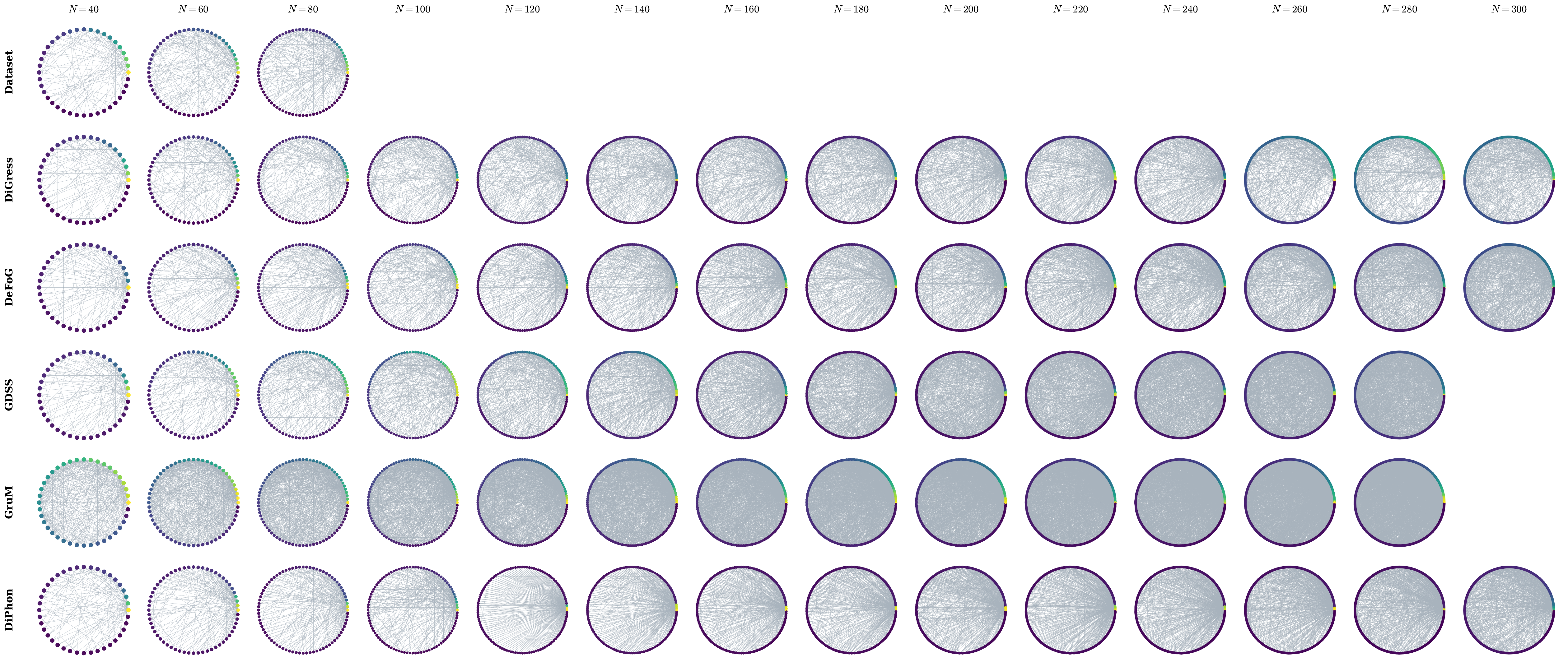}
\caption{Uncurated PA graph samples.}
\label{fig::app-pa-examples}
\end{figure}

\begin{figure}[H]
\centering
\includegraphics[width=0.9\linewidth]{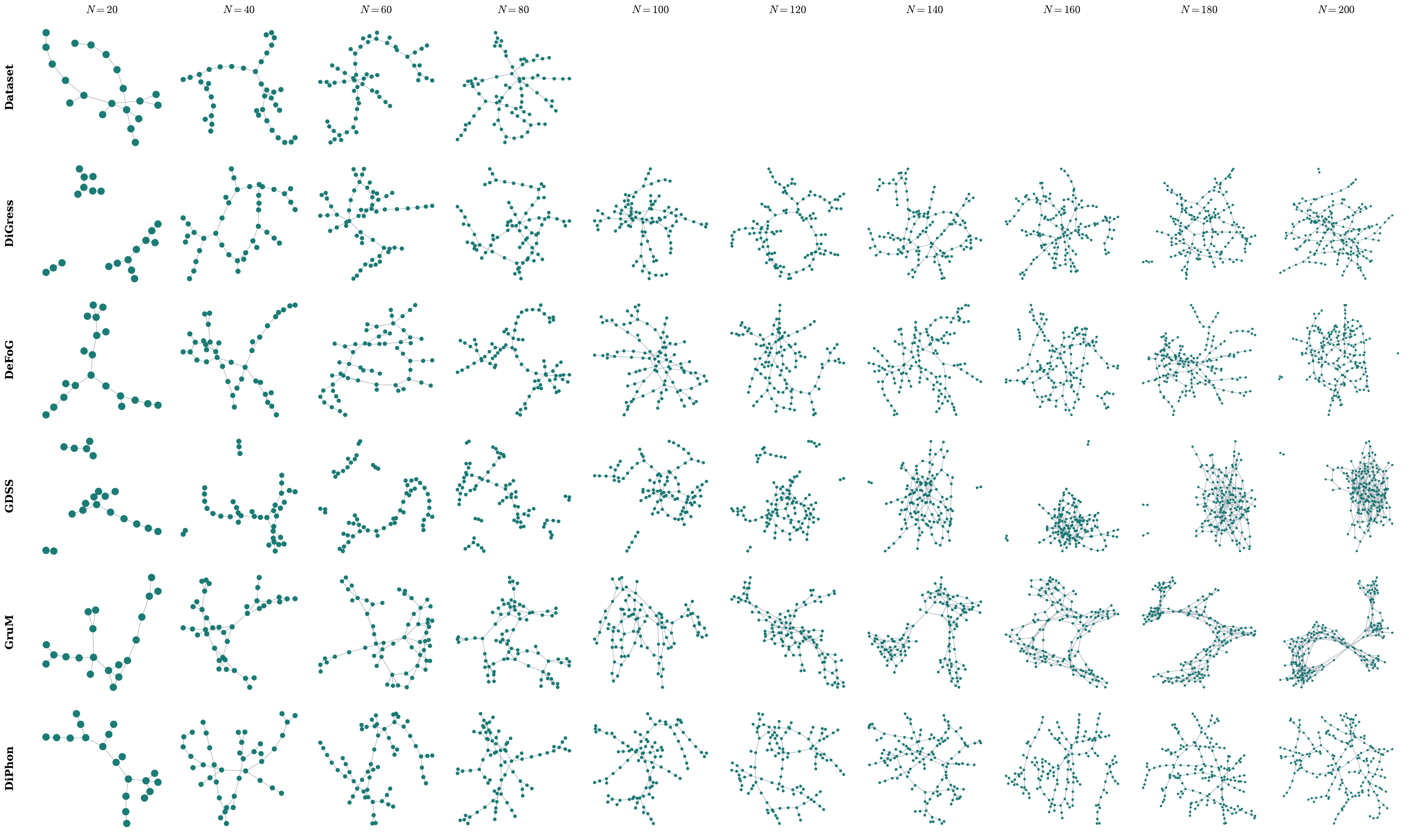}
\caption{Uncurated Tree graph samples.}
\label{fig::app-tree-examples}
\end{figure}

%% file: sections/appendix_discussion.tex
\section{Discussion and Limitations}
\label{app::lims}

This work connects two lines of research that have so far been treated separately: diffusion-based generative models on graphs and the theory of graphons as the limit object of dense graph sequences.
To the best of our knowledge, no previous diffusion model on graphs has placed the forward dynamics on graphon space.
Prior works that use graphons in a generative setting either estimate the graphon itself~\cite{xu2021graphon, xia2023implicit, azizpour2024scalable} or summarize it through motifs that serve as a prior for flow matching on graphs~\cite{wijesinghe2026flowette}.
The choice of the Jacobi process addresses a modelling issue that Gaussian SDEs leave open~\cite{jo2022gdss, jo2024grum}, since it keeps the trajectory inside the edge-probability range by construction.
The (approximate) moment matching ensures that, at the discrete level, we mimic the probability distributions of the continuous process, which is itself scale-free.
The convergence in homomorphism density to the Erd\H{o}s--R\'enyi reference provides a structural analogue of the role played by the Gaussian in score-matching approaches.
Beyond the theoretical analysis, the scalability experiments show that a model trained on small graphs transfers its sampling distribution to substantially larger graphs without retraining.

Regarding the main limitations of the ideas presented in this paper, first of all, the theoretical analysis is developed in the dense regime, since this is the regime in which the graphon is the canonical limit object of a sequence of growing graphs. 
This restriction belongs to the underlying mathematical object and not to the model. 
Empirically, the scalability experiments include sparse graph families as well, and the model transfers stably to large sizes in that regime, which suggests that the empirical scope of the approach exceeds that of its formal justification. 
A second limitation is that the equivalence between the two discretization routes, diffuse-then-discretize and discretize-then-diffuse, is established at the level of the first two moments. 
The second moment is controlled through a closed-form discrepancy that vanishes in the stationary regime, but the full distance between the marginal distributions is not bounded. 
For generative sampling it is plausible that this distance matters more than the first two moments, and characterizing it lies outside the scope of this work.
Finally, the denoising loss requires simulating a forward trajectory per sample at every training step, since the transition of the renormalized Jacobi process does not admit a closed-form density that would allow direct sampling at arbitrary times. 
This introduces a training overhead with respect to discrete and Gaussian-based approaches whose transition is tractable.
We also note that this work focuses on the generation of graphs with binary edges. From a practical standpoint, however, the framework can be naturally extended in several directions: i) node features can be modeled using the same class of Jacobi dynamics; ii) multiple edge categories can be handled through a stick-breaking construction~\cite{avdeyev2023dirichlet}; and iii) weighted graphs can be accommodated by applying bounded continuous diffusions to edge weights. These extensions are outside the scope of the present work and should be understood as modeling possibilities rather than fundamental limitations. However, the theoretical guarantees developed in this paper are specific to the binary-edge graphon setting and do not directly carry over to these variants.

%% file: references.bib
@book{lovasz2012large,
  title={Large networks and graph limits},
  author={Lov{\'a}sz, L{\'a}szl{\'o}},
  volume={60},
  year={2012},
  publisher={American Mathematical Soc.}
}

@article{lovasz2006limits,
  title={Limits of dense graph sequences},
  author={Lov{\'a}sz, L{\'a}szl{\'o} and Szegedy, Bal{\'a}zs},
  journal={Journal of Combinatorial Theory, Series B},
  volume={96},
  number={6},
  pages={933--957},
  year={2006},
  publisher={Elsevier}
}

@article{borgs2008convergent,
  title={Convergent sequences of dense graphs I: Subgraph frequencies, metric properties and testing},
  author={Borgs, Christian and Chayes, Jennifer T and Lov{\'a}sz, L{\'a}szl{\'o} and S{\'o}s, Vera T and Vesztergombi, Katalin},
  journal={Advances in Mathematics},
  volume={219},
  number={6},
  pages={1801--1851},
  year={2008},
  publisher={Elsevier}
}

@article{borgs2012convergent,
  title={Convergent sequences of dense graphs II. Multiway cuts and statistical physics},
  author={Borgs, Christian and Chayes, Jennifer T and Lov{\'a}sz, L{\'a}szl{\'o} and S{\'o}s, Vera T and Vesztergombi, Katalin},
  journal={Annals of Mathematics},
  pages={151--219},
  year={2012},
  publisher={JSTOR}
}

@article{holland1983sbm,
  title={Stochastic blockmodels: First steps},
  author={Holland, Paul W and Laskey, Kathryn Blackmond and Leinhardt, Samuel},
  journal={Social networks},
  volume={5},
  number={2},
  pages={109--137},
  year={1983},
  publisher={Elsevier}
}

@inproceedings{airoldi2013stochastic,
  title={Stochastic blockmodel approximation of a graphon: Theory and consistent estimation},
  author={Airoldi, Edo M and Costa, Thiago B and Chan, Stanley H},
  booktitle={Advances in Neural Information Processing Systems},
  volume={26},
  year={2013}
}

@inproceedings{ruiz2020graphon,
  title={Graphon neural networks and the transferability of graph neural networks},
  author={Ruiz, Luana and Chamon, Luiz and Ribeiro, Alejandro},
  booktitle={Advances in Neural Information Processing Systems},
  volume={33},
  pages={1702--1712},
  year={2020}
}

@article{ruiz2021graphonsp,
  title={Graphon signal processing},
  author={Ruiz, Luana and Chamon, Luiz FO and Ribeiro, Alejandro},
  journal={IEEE Transactions on Signal Processing},
  volume={69},
  pages={4961--4976},
  year={2021},
  publisher={IEEE}
}

@article{levie2021transferability,
  title={Transferability of spectral graph convolutional neural networks},
  author={Levie, Ron and Huang, Wei and Bucci, Lorenzo and Bronstein, Michael and Kutyniok, Gitta},
  journal={Journal of Machine Learning Research},
  volume={22},
  number={272},
  pages={1--59},
  year={2021}
}

@article{maskey2023transferability,
  title={Transferability of graph neural networks: an extended graphon approach},
  author={Maskey, Sohir and Levie, Ron and Kutyniok, Gitta},
  journal={Applied and Computational Harmonic Analysis},
  volume={63},
  pages={48--83},
  year={2023},
  publisher={Elsevier}
}

@article{xu2021graphon,
  title={Learning graphon autoencoders for generative graph modeling},
  author={Xu, Hongteng and Zhao, Peilin and Huang, Junzhou and Luo, Dixin},
  journal={arXiv preprint arXiv:2105.14244},
  year={2021}
}

@inproceedings{xia2023implicit,
  title={Implicit graphon neural representation},
  author={Xia, Xinyue and Mishne, Gal and Wang, Yusu},
  booktitle={International Conference on Artificial Intelligence and Statistics},
  pages={10619--10634},
  year={2023}}

@inproceedings{azizpour2024scalable,
  title={Scalable Implicit Graphon Learning},
  author={Azizpour, Ali and Zilberstein, Nicolas and Segarra, Santiago},
  booktitle={International Conference on Artificial Intelligence and Statistics},
  pages={3952--3960},
  year={2025}
}

@article{wijesinghe2026flowette,
  title={Flowette: Flow Matching with Graphette Priors for Graph Generation},
  author={Wijesinghe, Asiri and Kandanaarachchi, Sevvandi and Steinberg, Daniel M and Ong, Cheng Soon},
  journal={arXiv preprint arXiv:2602.23566},
  year={2026}
}

@inproceedings{vignac2023digress,
  title={DiGress: Discrete Denoising diffusion for graph generation},
  author={Vignac, Cl{\'e}ment and Krawczuk, Igor and Siraudin, Antoine and Wang, Bohan and Cevher, Volkan and Frossard, Pascal},
  booktitle={International Conference on Learning Representations},
  year={2023}
}

@inproceedings{qin2025defog,
  title={DeFoG: Discrete Flow Matching for Graph Generation},
  author={Qin, Yiming and Madeira, Manuel and Thanou, Dorina and Frossard, Pascal},
  booktitle={International Conference on Machine Learning},
  pages={50269--50326},
  year={2025}
}

@inproceedings{ma2023grit,
  title={Graph inductive biases in transformers without message passing},
  author={Ma, Liheng and Lin, Chen and Lim, Derek and Romero-Soriano, Adriana and Dokania, Puneet K and Coates, Mark and Torr, Philip and Lim, Ser-Nam},
  booktitle={International Conference on Machine Learning},
  pages={23321--23337},
  year={2023}}

@article{siraudin2024cometh,
  title={Cometh: A continuous-time discrete-state graph diffusion model},
  author={Siraudin, Antoine and Malliaros, Fragkiskos and Morris, Christopher},
  journal={Transactions of Machine Learning Research},
  year={2025}
}

@inproceedings{haefeli2022discrete,
  title={Diffusion Models for Graphs Benefit From Discrete State Spaces},
  author={Haefeli, Kilian Konstantin and Martinkus, Karolis and Perraudin, Nathana{\"e}l and Wattenhofer, Roger},
  booktitle={The First Learning on Graphs Conference},
  year={2022}
}

@inproceedings{austin2021d3pm,
  title={Structured denoising diffusion models in discrete state-spaces},
  author={Austin, Jacob and Johnson, Daniel D and Ho, Jonathan and Tarlow, Daniel and Van Den Berg, Rianne},
  booktitle={Advances in neural information processing systems},
  volume={34},
  pages={17981--17993},
  year={2021}
}

@inproceedings{jo2022gdss,
  title={Score-based generative modeling of graphs via the system of stochastic differential equations},
  author={Jo, Jaehyeong and Lee, Seul and Hwang, Sung Ju},
  booktitle={International conference on machine learning},
  pages={10362--10383},
  year={2022}}

@inproceedings{niu2020edpgnn,
  title={Permutation invariant graph generation via score-based generative modeling},
  author={Niu, Chenhao and Song, Yang and Song, Jiaming and Zhao, Shengjia and Grover, Aditya and Ermon, Stefano},
  booktitle={International conference on artificial intelligence and statistics},
  pages={4474--4484},
  year={2020}}

@inproceedings{jo2024grum,
  title={Graph Generation with Diffusion Mixture},
  author={Jo, Jaehyeong and Kim, Dongki and Hwang, Sung Ju},
  booktitle={International Conference on Machine Learning},
  pages={22371--22405},
  year={2024}
}

@inproceedings{avdeyev2023dirichlet,
  title={Dirichlet diffusion score model for biological sequence generation},
  author={Avdeyev, Pavel and Shi, Chenlai and Tan, Yuhao and Dudnyk, Kseniia and Zhou, Jian},
  booktitle={International Conference on Machine Learning},
  pages={1276--1301},
  year={2023}}

@inproceedings{stark2024dirichlet,
  title={Dirichlet Flow Matching with Applications to DNA Sequence Design},
  author={Stark, Hannes and Jing, Bowen and Wang, Chenyu and Corso, Gabriele and Berger, Bonnie and Barzilay, Regina and Jaakkola, Tommi},
  booktitle={International Conference on Machine Learning},
  pages={46495--46513},
  year={2024}
}

@inproceedings{eijkelboom2024catflow,
  title={Variational flow matching for graph generation},
  author={Eijkelboom, Floor and Bartosh, Grigory and Naesseth, Christian A and Welling, Max and van de Meent, Jan-Willem},
  booktitle={Advances in Neural Information Processing Systems},
  volume={37},
  pages={11735--11764},
  year={2024}
}

@article{fishman2023constrained,
  title={Diffusion models for constrained domains},
  author={Fishman, Nic and Klarner, Leo and De Bortoli, Valentin and Mathieu, Emile and Hutchinson, Michael},
  journal={arXiv preprint arXiv:2304.05364},
  year={2023}
}

@inproceedings{debortoli2022riemannian,
  title={Riemannian score-based generative modelling},
  author={De Bortoli, Valentin and Mathieu, Emile and Hutchinson, Michael and Thornton, James and Teh, Yee Whye and Doucet, Arnaud},
  booktitle={Advances in neural information processing systems},
  volume={35},
  pages={2406--2422},
  year={2022}
}

@inproceedings{yim2023se3,
  title={SE (3) diffusion model with application to protein backbone generation},
  author={Yim, Jason and Trippe, Brian L and De Bortoli, Valentin and Mathieu, Emile and Doucet, Arnaud and Barzilay, Regina and Jaakkola, Tommi},
  booktitle={International Conference on Machine Learning},
  pages={40001--40039},
  year={2023}
}

@inproceedings{lou2024sedd,
  title={Discrete Diffusion Modeling by Estimating the Ratios of the Data Distribution},
  author={Lou, Aaron and Meng, Chenlin and Ermon, Stefano},
  booktitle={International Conference on Machine Learning},
  pages={32819--32848},
  year={2024}
}

@article{anderson1982reverse,
  title={Reverse-time diffusion equation models},
  author={Anderson, Brian},
  journal={Stochastic Processes and their Applications},
  volume={12},
  number={3},
  pages={313--326},
  year={1982},
  publisher={Elsevier}
}

@inproceedings{sohldickstein2015deep,
  title={Deep unsupervised learning using nonequilibrium thermodynamics},
  author={Sohl-Dickstein, Jascha and Weiss, Eric and Maheswaranathan, Niru and Ganguli, Surya},
  booktitle={International conference on machine learning},
  pages={2256--2265},
  year={2015}}

@inproceedings{ho2020ddpm,
  title={Denoising diffusion probabilistic models},
  author={Ho, Jonathan and Jain, Ajay and Abbeel, Pieter},
  booktitle={Advances in neural information processing systems},
  volume={33},
  pages={6840--6851},
  year={2020}
}

@inproceedings{song2019generative,
  title={Generative modeling by estimating gradients of the data distribution},
  author={Song, Yang and Ermon, Stefano},
  booktitle={Advances in neural information processing systems},
  volume={32},
  year={2019}
}

@inproceedings{song2021scoresde,
  title={Score-Based Generative Modeling through Stochastic Differential Equations},
  author={Song, Yang and Sohl-Dickstein, Jascha and Kingma, Diederik P and Kumar, Abhishek and Ermon, Stefano and Poole, Ben},
  booktitle={International Conference on Learning Representations},
  year={2021}
}

@article{hyvarinen2005score,
  title={Estimation of non-normalized statistical models by score matching.},
  author={Hyv{\"a}rinen, Aapo and Dayan, Peter},
  journal={Journal of Machine Learning Research},
  volume={6},
  number={4},
  year={2005}
}

@article{vincent2011denoising,
  title={A connection between score matching and denoising autoencoders},
  author={Vincent, Pascal},
  journal={Neural computation},
  volume={23},
  number={7},
  pages={1661--1674},
  year={2011},
  publisher={MIT Press}
}

@book{karlin1981second,
  title={A second course in stochastic processes},
  author={Karlin, Samuel and Taylor, Howard E},
  year={1981},
  publisher={Elsevier}
}

@article{wright1931evolution,
  title={Evolution in Mendelian populations},
  author={Wright, Sewall},
  journal={Genetics},
  volume={16},
  number={2},
  pages={97},
  year={1931}
}

@book{etheridge2011mathematical,
  title={Some Mathematical Models from Population Genetics: {\'E}cole D'{\'E}t{\'e} de Probabilit{\'e}s de Saint-Flour XXXIX-2009},
  author={Etheridge, Alison},
  volume={2012},
  year={2011},
  publisher={Springer Science \& Business Media}
}

@article{forman2008pearson,
  title={The Pearson diffusions: A class of statistically tractable diffusion processes},
  author={Forman, Julie Lyng and S{\o}rensen, Michael},
  journal={Scandinavian Journal of Statistics},
  volume={35},
  number={3},
  pages={438--465},
  year={2008},
  publisher={Wiley Online Library}
}

@article{demni2009large,
  title={Large deviations for statistics of the Jacobi process},
  author={Demni, Nizar and Zani, Marguerite},
  journal={Stochastic Processes and their Applications},
  volume={119},
  number={2},
  pages={518--533},
  year={2009},
  publisher={Elsevier}
}

@article{bibby2005diffusion,
  title={Diffusion-type models with given marginal distribution and autocorrelation function},
  author={Bibby, Bo Martin and Skovgaard, Ib Michael and S{\o}rensen, Michael},
  journal={Bernoulli},
  volume={11},
  number={2},
  pages={191--220},
  year={2005},
  publisher={Bernoulli Society for Mathematical Statistics and Probability}
}

@article{gourieroux2006multivariate,
  title={Multivariate Jacobi process with application to smooth transitions},
  author={Gourieroux, Christian and Jasiak, Joann},
  journal={Journal of econometrics},
  volume={131},
  number={1-2},
  pages={475--505},
  year={2006},
  publisher={Elsevier}
}

@book{szego1939orthogonal,
  title={Orthogonal polynomials},
  author={Szeg, Gabor},
  volume={23},
  year={1939},
  publisher={American Mathematical Soc.}
}

@book{bakry2014analysis,
  title={Analysis and geometry of Markov diffusion operators},
  author={Bakry, Dominique and Gentil, Ivan and Ledoux, Michel and others},
  volume={103},
  year={2014},
  publisher={Springer}
}

@inproceedings{vignac2023midi,
  title={Midi: Mixed graph and 3d denoising diffusion for molecule generation},
  author={Vignac, Clement and Osman, Nagham and Toni, Laura and Frossard, Pascal},
  booktitle={Joint European Conference on Machine Learning and Knowledge Discovery in Databases},
  pages={560--576},
  year={2023}}

@article{kloeden1992numerical,
  title={The numerical solution of stochastic differential equations},
  author={Kloeden, Peter E and Pearson, RA},
  journal={The ANZIAM Journal},
  volume={20},
  number={1},
  pages={8--12},
  year={1977},
  publisher={Cambridge University Press}
}

@inproceedings{karras2022elucidating,
  title={Elucidating the design space of diffusion-based generative models},
  author={Karras, Tero and Aittala, Miika and Aila, Timo and Laine, Samuli},
  booktitle={Advances in neural information processing systems},
  volume={35},
  pages={26565--26577},
  year={2022}
}

@article{mcdiarmid1989method,
  title={On the method of bounded differences},
  author={McDiarmid, Colin and others},
  journal={Surveys in combinatorics},
  volume={141},
  number={1},
  pages={148--188},
  year={1989},
  publisher={Norwich}
}

@article{kandanaarachchi2024graphons,
  title={Graphons of line graphs},
  author={Kandanaarachchi, Sevvandi and Ong, Cheng Soon},
  journal={arXiv preprint arXiv:2409.01656},
  year={2024}
}

@inproceedings{martinkus2022spectre,
  title={Spectre: Spectral conditioning helps to overcome the expressivity limits of one-shot graph generators},
  author={Martinkus, Karolis and Loukas, Andreas and Perraudin, Nathana{\"e}l and Wattenhofer, Roger},
  booktitle={International Conference on Machine Learning},
  pages={15159--15179},
  year={2022}}

@inproceedings{davis2024fisher,
  title={Fisher flow matching for generative modeling over discrete data},
  author={Davis, Oscar and Kessler, Samuel and Petrache, Mircea and Ceylan, {\.I}smail {\.I} and Bronstein, Michael and Bose, Avishek J},
  booktitle={Advances in Neural Information Processing Systems},
  volume={37},
  pages={139054--139084},
  year={2024}
}

@article{li2023graphpnas,
  title={GraphPNAS: Learning probabilistic graph generators for neural architecture search},
  author={Li, Muchen and Liu, Jeffrey Yunfan and Sigal, Leonid and Liao, Renjie},
  journal={Transactions on Machine Learning Research},
  year={2023}
}

@inproceedings{uslu2026graph,
  author={Uslu, Yiğit Berkay and Hadou, Samar and Rozada, Sergio and Bidokhti, Shirin Saeedi and Ribeiro, Alejandro},
  booktitle={IEEE International Conference on Acoustics, Speech and Signal Processing}, 
  title={Graph Signal Generative Diffusion Models}, 
  year={2026},
  volume={},
  number={},
  pages={626-630}}

@inproceedings{rozada2026grph,
  author={Rozada, Sergio and B, Vimal K and Cavallo, Andrea and Marques, Antonio G. and Jamali-Rad, Hadi and Isufi, Elvin},
  booktitle={IEEE International Conference on Acoustics, Speech and Signal Processing}, 
  title={Graph-Aware Diffusion for Signal Generation}, 
  year={2026},
  volume={},
  number={},
  pages={461-465}}

@inproceedings{wen2023diffstg,
  title={Diffstg: Probabilistic spatio-temporal graph forecasting with denoising diffusion models},
  author={Wen, Haomin and Lin, Youfang and Xia, Yutong and Wan, Huaiyu and Wen, Qingsong and Zimmermann, Roger and Liang, Yuxuan},
  booktitle={ACM international conference on advances in geographic information systems},
  pages={1--12},
  year={2023}
}
